\documentclass[a4paper,11pt]{article}
\pdfoutput=1

\usepackage{jheppub}
\usepackage[T1]{fontenc}
\usepackage{amsmath,amssymb,amsthm}
\usepackage{mathtools}
\usepackage{array}
\usepackage{etoolbox}
\usepackage{mleftright}
\mleftright
\usepackage{tikz}
\usetikzlibrary{arrows.meta}
\graphicspath{{figs/}}

\theoremstyle{definition}
\newtheorem{proposition}{Proposition}[section]

\newtheorem{remark}[proposition]{Remark}

\newcommand{\eps}{\varepsilon}
\newcommand{\R}{\mathbb{R}}
\newcommand{\E}{\mathbb{E}}
\newcommand{\cF}{\mathcal{F}}
\newcommand{\cG}{\mathcal{G}}
\newcommand{\hatQ}{\widehat{Q}}
\newcommand{\Cw}{C_{\mathrm{W}}}
\newcommand{\Cb}{C_{\mathrm{b}}}
\newcommand{\Keft}{K_{1,\mathrm{EFT}}}
\newcommand{\Kmic}{K_{1,\mathrm{mic}}}
\newcommand{\Uex}{U_{1,\mathrm{exact}}}
\newcommand{\Umod}{U_{1,\mathrm{model}}}
\newcommand{\Hess}{D^2Q}

\AtBeginEnvironment{table}{\setlength{\belowcaptionskip}{6pt}}

\tikzset{
  resp/.style={-{Latex[length=2.0mm]}, line width=.45pt},
  obs/.style={circle, draw, fill=black, inner sep=0pt, minimum size=2.2pt},
  bulk/.style={rectangle, draw, fill=black, inner sep=1.2pt, minimum size=3.2pt},
  bblob/.style={circle, draw, line width=.45pt, fill=gray!12,
                minimum size=7.5mm, inner sep=1pt},
  dvert/.style={circle, draw, line width=.45pt, fill=white,
                minimum size=1.2mm, inner sep=0pt},
  lab/.style={font=\scriptsize, inner sep=1pt},
  tlab/.style={font=\tiny, inner sep=1pt}
}
\newcommand{\diagVFourTransport}{%
  \vcenter{\hbox{%
  \begin{tikzpicture}[baseline=-0.65ex, x=1em, y=1em]
    \node[bblob] (B) at (0,0) {$\scriptstyle V_4^0$};
    \node[obs] (T1) at (3.0,0.55) {};
    \node[obs] (T2) at (3.0,-0.55) {};
    \draw[resp] (B.east) -- (T1);
    \draw[resp] (B.east) -- (T2);
    \node[tlab] at (0,-1.6) {$0$};
    \node[tlab] at (3.0,-1.6) {$t$};
  \end{tikzpicture}}}%
}
\newcommand{\diagVFourSource}{%
  \vcenter{\hbox{%
  \begin{tikzpicture}[baseline=-0.65ex, x=1em, y=1em]
    \node[bulk] (S) at (0,0) {};
    \node[inner sep=1pt] at (-0.9,0) {$\Sigma$};
    \node[obs] (T1) at (3.0,0.55) {};
    \node[obs] (T2) at (3.0,-0.55) {};
    \draw[resp] (S.east) -- (T1);
    \draw[resp] (S.east) -- (T2);
    \node[tlab] at (0,-1.6) {$u$};
    \node[tlab] at (3.0,-1.6) {$t$};
  \end{tikzpicture}}}%
}
\newcommand{\diagKOneSource}{%
  \vcenter{\hbox{%
  \begin{tikzpicture}[baseline=-0.75ex, x=1em, y=1em]
    \node[dvert] (V) at (0,0) {};
    \node[obs]   (T) at (3.1,0) {};
    \draw[resp] (V.east) -- (T);
    \draw[line width=.45pt] (V) to[out=135,in=45,looseness=35]
      node[lab, above] {$V_4$} (V);
    \node[lab]  at (-1.7,0) {$\frac12\,\Hess$};
    \node[tlab] at (0,-0.7) {$s$};
    \node[tlab] at (3.1,-0.7) {$t$};
  \end{tikzpicture}}}%
}
\newcommand{\diagKOneTransport}{%
  \vcenter{\hbox{%
  \begin{tikzpicture}[baseline=-0.65ex, x=1em, y=1em]
    \node[bblob] (B) at (0,0) {$\scriptstyle K_1^0$};
    \node[obs]   (T) at (3.0,0) {};
    \draw[resp] (B.east) -- (T);
    \node[tlab] at (0,-1.6) {$0$};
    \node[tlab] at (3.0,-1.6) {$t$};
  \end{tikzpicture}}}%
}

\title{\boldmath Collective Kernel EFT for Pre-activation ResNets}

\author{Hidetoshi Kawase}
\author{and Toshihiro Ota}
\affiliation{CyberAgent, Inc., Shibuya, Tokyo 150--6121 Japan}
\emailAdd{kawase\_hidetoshi@cyberagent.co.jp}
\emailAdd{ota\_toshihiro@cyberagent.co.jp}

\abstract{%
In finite-width deep neural networks, the empirical kernel $G$ evolves stochastically across layers.
We develop a collective kernel effective field theory (EFT) for a simplified pre-activation residual model.
Exact conditional Gaussianity of the residual increment gives finite-width discrete identities for $G$ and its sigma-kernel source.
A projected $G$-only closure and a continuous-depth approximation define reference equations for the mean kernel $K_0$, the kernel covariance $V_4$, and the next-to-leading mean correction $K_1$.
Numerical tests show that this closure captures the leading mean-kernel evolution, whereas its covariance and next-to-leading accuracy depend on sigma-kernel modes omitted from the $G$-only state.
For the scalar ($N=1$) model with $\tanh$ activation, a Hermite--Galerkin extension agrees with the extrapolated $K_1$ value and reduces the $V_4$ endpoint error by $85$--$93\%$ relative to the $G$-only reference.
These results clarify the range of validity of the equal-time Markov $G$-only description and identify the additional collective information required beyond it.
}

\makeatletter
\gdef\@fpheader{ \vspace{1em} }
\makeatother

\begin{document}
\maketitle
\flushbottom

\section{Introduction}

Developing a systematic theory of finite-width effects in deep neural networks---beyond the infinite-width Gaussian process limit~\cite{neal1996bayesian,williams1996computing,lee2018dnn_gp,matthews2018gp_behaviour} and the Neural Tangent Kernel~\cite{jacot2018ntk}---has been an active research area~\cite{yaida2020nongaussian,roberts2022principles,hanin2022correlation_functions}.
Field-theoretic formulations of network ensembles have been developed in parallel, both as a systematic statistical field theory for network dynamics~\cite{helias2020statistical,segadlo2022unified} and as a correspondence between network ensembles and (non-)Gaussian field theories~\cite{halverson2021neural_qft,grosvenor2022edge}.
Banta et al.~\cite{banta2024structures} constructed a diagrammatic EFT for MLPs, organizing $1/n$ corrections as a field theory via susceptibility structure and ghost fields.
We extend that perspective to pre-activation ResNets~\cite{he2016deepresidual,he2016identity}, adopting a residual scaling $\eps$ under which the kernel drift is $O(\eps^2)$---consistent with the stable-ResNet regime analyzed in~\cite{hayou2021stable_resnet}.
The continuous-depth interpretation of residual architectures, in which a block update corresponds to one step of an ordinary differential equation, underlies the neural ODE literature~\cite{e2017proposal,haber2017stable,chen2018neural_ode}.
In complementary work, Guillen et al.~\cite{guillen2025finite_width_ntk} developed a diagrammatic framework for finite-width NTK statistics and leading-order training dynamics in MLPs.

Signal propagation for ResNets has been studied via mean-field theory~\cite{yang2017mean_field_resnet,schoenholz2017deep_information,poole2016transient_chaos} and the infinite depth-and-width limit~\cite{li2021future_log_gaussian,li2022neural_covariance_sde,hayou2023width_depth_commute,peluchetti2020diffusion,peluchetti2021doubly_infinite}, with field-theoretic formulations also emerging~\cite{littwin2021random_kernels,fischer2023field_theory_resnets}.
Finite-width corrections to kernel statistics have also been analyzed through dynamical mean-field theory, which tracks kernel evolution and its fluctuations across random initializations~\cite{bordelon2022self_consistent,bordelon2023finite_width}, and through finite depth-and-width expansions of the tangent kernel and of Bayesian posteriors~\cite{hanin2020finite_depth_width,zavatoneveth2021asymptotics}.
This paper contributes a systematic description of finite-width corrections from the perspective of a collective kernel stochastic EFT.

For a fixed set of inputs, each entry $G_{ab}^\ell$ is the width-normalized inner product between the hidden representations of inputs $a$ and $b$ at depth $\ell$.
Finite-width corrections to $G^\ell$ therefore describe departures of the initialization-time representation geometry from its deterministic infinite-width limit, rather than only corrections to an abstract field-theoretic observable.
We organize the leading finite-width statistics as
\begin{equation}
  \E[G^\ell]
  =
  K_0^\ell+\frac{1}{n}K_1^\ell+o(n^{-1}),
  \qquad
  V_4^\ell
  :=
  n\,\mathrm{Cov}(G^\ell,G^\ell).
  \label{eq:finite-width-representation-statistics}
\end{equation}
Thus $K_1^\ell$ determines the leading systematic shift of the ensemble-mean Gram matrix, whereas $V_4^\ell$ determines its leading covariance across random initializations.
These effects are complementary: an ensemble may have an accurate mean kernel but non-negligible run-to-run variability, or a small variance but a systematically shifted mean.
Our aim is to quantify the finite-width information lost when this geometry is reduced to $K_0$ alone and to identify the additional collective statistics required beyond it.

Both ResNets and MLPs admit an exact conditional-Gaussian layer law at each block.
The key difference is which variable is the natural Gaussian field:
in MLPs (as in Banta et al.~\cite{banta2024structures}), the preactivation $\phi^{\ell+1}$ given $\phi^\ell$ is directly the Gaussian variable;
in ResNets, it is the increment $\eta^\ell$ that is the natural conditional-Gaussian variable.
Using a Fourier representation of the normalized conditional Gaussian transition yields an exact response-field action without auxiliary Grassmann fields.
Working with the increment separates deterministic skip carry-over from stochastic residual noise, exposes the natural continuous-depth scaling, and makes clear which kernel relations are exact and which rely on effective approximations.

Banta et al.\ formulate an MLP theory in the preactivation variables, whereas we derive an increment-based finite-width block law for residual dynamics and formulate the collective theory directly in the empirical kernel $G$.
Relative to infinite depth-and-width ResNet limits, we retain the leading $1/n$ covariance and mean corrections at fixed residual time.
We also identify the matrix-valued $G$-only closure defects directly and develop a scalar sigma-kernel extension that predicts the corresponding corrections.

\subsection{Setup and Notation}\label{sec:setup}

We fix input points $x_1,\dots,x_N$ and use indices $a,b,c,d=1,\dots,N$ (input points),
$i,j=1,\dots,n$ (neurons), and $\ell=0,1,\dots,L$ (depth).
The preactivation of neuron $i$ at depth $\ell$ is written
$\phi_i^\ell = (\phi_i^\ell(a))_{a=1}^N \in \R^N$.

One block of the pre-activation ResNet~\cite{he2016identity} is defined as
\begin{align}
  \phi_i^{\ell+1}(a) &= \phi_i^\ell(a) + \eps\,\eta_i^\ell(a), \label{eq:resnet-block}\\
  \eta_i^\ell(a)     &= \sum_{j=1}^n W_{ij}^\ell\,\sigma(\phi_j^\ell(a)) + b_i^\ell \notag
\end{align}
where weights and biases are drawn from independent Gaussian initializations
$W_{ij}^\ell \sim \mathcal{N}(0,\Cw/n)$ and $b_i^\ell \sim \mathcal{N}(0,\Cb)$.
This is a simplified fully connected pre-activation residual block; convolution, normalization layers, and post-activation variants are not included in the model studied here.
Here $\eps>0$ is the residual scaling parameter; in the continuous-depth limit we set $dt=\eps^2$
(based on the scaling that the conditional mean change of the kernel is $O(\eps^2)$ and finite-width fluctuations are $O(\eps/\sqrt{n})$;
see also~\cite{hayou2021stable_resnet} for a related stability analysis of this scaling).
The initial preactivations are drawn as i.i.d.\ centered Gaussians:
\begin{equation}
  \phi_i^0 \stackrel{\text{i.i.d.}}{\sim} \mathcal{N}(0, K_0^0), \qquad i = 1,\dots,n,
  \label{eq:phi0-init}
\end{equation}
where $K_0^0 \in \mathrm{Sym}_N^{++}$ is a given initial covariance matrix.
Under this assumption, $\bar{K}^0 = K_0^0$, $V_4^0 = O(1)$ (Wishart fluctuations), and the exact identity $\Uex^0=0$ stated below holds.
The exact conditional-Gaussian block law below requires only the relevant activation moments and therefore includes nonsmooth choices such as ReLU.
The derivative-based $G$-only NLO equations require additional smoothness.
The exact identities and the $G$-only diagnostics are formulated for general matrix-valued kernels.
The constructive extension developed in Sec.~\ref{sec:hermite} is restricted to the scalar diagonal setting with $N=1$ and $\sigma(x)=\tanh x$.

\subsection{Summary of Main Results}

The main results of this paper are summarized in four points.

\begin{itemize}
  \item \textbf{Exact conditional-Gaussian block law and response-field representation.}
    Conditioning on $\phi^\ell$, the increments $\eta_i^\ell$ are exactly Gaussian,
    with conditional covariance $\hatQ_{ab}^\ell = \Cb + \frac{\Cw}{n}\sum_j \sigma(\phi_j^\ell(a))\sigma(\phi_j^\ell(b))$.
    The Fourier representation gives an exact discrete response-field action without introducing auxiliary Grassmann fields (Section~\ref{sec:block}).

  \item \textbf{Exact kernel recursion and microscopic source.}
    The empirical kernel $G_{ab}^\ell := n^{-1}\sum_i \phi_i^\ell(a)\phi_i^\ell(b)$ updates exactly as $G^{\ell+1} = G^\ell + \eps H^\ell + \eps^2 J^\ell$,
    and $\Uex^\ell = n(\bar{S}^\ell - E_2(K_0^\ell))$ is obtained as an exact identity around the chosen background $K_0$;
    this serves as the exact source for $\Kmic$ and is compared with the model source $\Umod^\ell$ in Section~\ref{sec:source-diag}.

  \item \textbf{Hierarchy derivation for a reference $G$-only closure.}
    The continuous-depth $K_0$ equation follows from the leading $G$-state reduction (GC0).
    The $V_4$ baseline additionally uses the covariance-level projection (PROJ) and a leading background evaluation of the exact noise source.
    The reference $\Keft$ equation further uses a second-order delta expansion and the NLO conditional-bias closure (GC1).
    Here $V_4 \equiv V_4^{(G)}$ denotes the kernel fluctuation covariance, which in the MLP comparison differs from the preactivation connected four-point function $V_4^{(\phi)}$ of Banta et al.~\cite{banta2024structures} by a conditional-Wishart term (see Remark~\ref{rem:bridge}).
    $\Keft$ is reinterpreted diagrammatically as the one-loop tadpole of the drift cubic vertex (Sections~\ref{sec:kernel}--\ref{sec:diag}).

  \item \textbf{Validity boundaries and a constructive scalar extension.}
    While $K_0$ is well reproduced, width extrapolation in two tested matrix components localizes the persistent $V_4$ residual to the conditional sigma-kernel projection contribution.
    For $K_1$, the leading conditional-Wishart term cancels the spurious initial tadpole and improves the reference trajectory, but the finite-width measurements retain a non-Wishart remainder at later depths.
    A finite Hermite--Galerkin closure for scalar diagonal $\tanh$ then improves $K_1$, $U_1$, and the endpoint $V_4$ prediction (Sections~\ref{sec:numerics} and~\ref{sec:hermite}).
    A matched $\alpha=0$ MLP control is consistent with vanishing large-width residuals, supporting a same-neuron memory interpretation without constituting a general architecture-classification theorem.
\end{itemize}

The remainder of this paper is organized as follows.
Section~\ref{sec:block} derives the exact block law.
Section~\ref{sec:kernel} derives the exact recursion and the projected $G$-only baseline equations for $K_0$, $V_4$, and $\Keft$.
Section~\ref{sec:eft} constructs the corresponding reference collective bilocal EFT and provides its diagrammatic interpretation.
Section~\ref{sec:numerics} presents the $G$-only validity diagnostics;
Section~\ref{sec:hermite} develops and validates the constructive scalar extension;
Section~\ref{sec:disc} discusses machine-learning significance, scope, and future directions;
and Section~\ref{sec:conc} concludes.
A notation guide is provided in Appendix~\ref{sec:notation-guide}.
Code reproducing all numerical figures is available at \url{https://github.com/kavvase/resnet_eft}.

\section{Exact One-Block Law}\label{sec:block}

\subsection{Exact Conditional Gaussian Law of the Increment}

\begin{proposition}[Conditional distribution of the increment]\label{prop:increment-gaussian}
  Conditioning on $\phi^\ell$, the increments $\eta_1^\ell, \dots, \eta_n^\ell$ are
  conditionally independent $N$-dimensional Gaussian vectors:
  \begin{equation}
    \eta_i^\ell \mid \phi^\ell \sim \mathcal{N}(0,\, \cG_\eta^\ell[\phi^\ell]),
    \label{eq:increment-gaussian}
  \end{equation}
  \begin{equation}
    \cG_{\eta,ab}^\ell[\phi^\ell]
    := \Cb + \frac{\Cw}{n}\sum_{j=1}^n \sigma(\phi_j^\ell(a))\sigma(\phi_j^\ell(b)).
    \label{eq:G-eta}
  \end{equation}
  This identity is exact for finite $n$ and finite $N$.
\end{proposition}

\begin{proof}
Conditional on $\phi^\ell$,
$\eta_i^\ell(a) = \sum_{j=1}^n W_{ij}^\ell\,\sigma(\phi_j^\ell(a)) + b_i^\ell$
is a linear combination of independent Gaussian variables and hence Gaussian.
The conditional mean vanishes since $\E[W_{ij}^\ell] = \E[b_i^\ell] = 0$.
The conditional covariance follows from
$\E[W_{ij}^\ell W_{ik}^\ell] = (\Cw/n)\delta_{jk}$ and $\E[(b_i^\ell)^2] = \Cb$:
\begin{align}
  \E[\eta_i^\ell(a)\,\eta_i^\ell(b) \mid \phi^\ell]
  = \frac{\Cw}{n}\sum_{j=1}^n \sigma(\phi_j^\ell(a))\sigma(\phi_j^\ell(b)) + \Cb
  = \cG_{\eta,ab}^\ell[\phi^\ell].
\end{align}
Different neurons $i \neq i'$ use independent weights and are therefore conditionally independent.
\end{proof}

The joint measure of the block can thus be written as
\begin{equation}
  P(\phi^{\ell+1}, \eta^\ell \mid \phi^\ell)
  = \prod_{i=1}^n\prod_{a=1}^N \delta\bigl(\phi_i^{\ell+1}(a) - \phi_i^\ell(a) - \eps\,\eta_i^\ell(a)\bigr)
    \cdot P(\eta^\ell \mid \phi^\ell).
  \label{eq:joint-block}
\end{equation}

\subsection{Exact MSRJD Block Action}

Exponentiating each delta constraint via the Fourier representation $\delta(z) = \int \frac{d\hat\phi}{2\pi}\,e^{i\hat\phi z}$
and introducing response fields $\hat\phi_i^\ell \in \R^N$, we write the joint measure (\ref{eq:joint-block}) as
\begin{equation}
  P(\phi^{\ell+1}, \eta^\ell \mid \phi^\ell)
  \propto \int D\hat\phi^\ell\;\exp\bigl[-S_{\mathrm{mix}}^\ell\bigr],
\end{equation}
with the mixed-field action
\begin{align}
  S_{\mathrm{mix}}^\ell
  &= -i\sum_{i=1}^n (\hat\phi_i^\ell)^\top(\phi_i^{\ell+1} - \phi_i^\ell - \eps\,\eta_i^\ell) \notag\\
  &\quad + \frac{1}{2}\sum_{i=1}^n (\eta_i^\ell)^\top (\cG_\eta^\ell[\phi^\ell])^{-1} \eta_i^\ell
    + \frac{n}{2}\log\det(2\pi\,\cG_\eta^\ell[\phi^\ell]).
  \label{eq:mixed-action}
\end{align}
When $\cG_\eta^\ell$ is degenerate (e.g., $\Cb=0$ with linearly dependent inputs),
the Gaussian measure and the $\eta^\ell$ integral are restricted to
$\operatorname{Im}\cG_\eta^\ell$, with the null-direction constraint
$P_{\ker\cG_\eta^\ell}\eta_i^\ell=0$.
On this support, $(\cG_\eta^\ell)^{-1}$ and $\det\cG_\eta^\ell$ in the following formulas mean the Moore--Penrose inverse and pseudo-determinant.
Appendix~\ref{sec:appx-singular} gives the support-restricted measure explicitly.

Fixing $\phi^\ell$, $\phi^{\ell+1}$, and $\hat\phi^\ell$, completing the square in $\eta^\ell$
(writing $\cG \equiv \cG_\eta^\ell[\phi^\ell]$) gives
\begin{equation}
  \frac{1}{2}(\eta_i + i\eps\,\cG\hat\phi_i)^\top\cG^{-1}(\eta_i + i\eps\,\cG\hat\phi_i)
  + \frac{\eps^2}{2}(\hat\phi_i)^\top\cG\,\hat\phi_i.
\end{equation}
Performing the Gaussian integral over each $\eta_i$ produces a factor $[\det(2\pi\cG)]^{n/2}$,
which exactly cancels the normalization factor represented by the third term in the mixed-field action.
The action after integrating out $\eta^\ell$ is thus
\begin{equation}
  S_{\mathrm{MSR}}^\ell
  = -i\sum_{i=1}^n (\hat\phi_i^\ell)^\top(\phi_i^{\ell+1} - \phi_i^\ell)
    + \frac{\eps^2}{2}\sum_{i=1}^n (\hat\phi_i^\ell)^\top \cG_\eta^\ell[\phi^\ell]\, \hat\phi_i^\ell.
  \label{eq:MSRJD}
\end{equation}
This is the exact discrete Martin--Siggia--Rose--Janssen--De~Dominicis (MSRJD) block action~\cite{martin1973statistical,janssen1976lagrangean,dedominicis1976techniques}.
Because the determinant factor is absorbed into the increment integration, no ghost fields appear.

Finally, when $\cG_\eta^\ell$ is positive definite, integrating over
$\hat\phi^\ell$ yields the exact block transition density:
\begin{equation}
  P(\phi^{\ell+1} \mid \phi^\ell)
  = \bigl[\det(2\pi\eps^2\cG_\eta^\ell)\bigr]^{-n/2}
    \exp\left[-\frac{1}{2\eps^2}\sum_{i=1}^n
      (\phi_i^{\ell+1}-\phi_i^\ell)^\top (\cG_\eta^\ell)^{-1}
      (\phi_i^{\ell+1}-\phi_i^\ell)\right].
  \label{eq:block-kernel}
\end{equation}
Each neuron independently satisfies
$\phi_i^{\ell+1} - \phi_i^\ell \sim \mathcal{N}(0, \eps^2\cG_\eta^\ell[\phi^\ell])$,
and $dt=\eps^2$ gives the natural time scale for the continuous-depth limit.

\section{Exact Kernel Recursion and Gaussian Closure Hierarchy}\label{sec:kernel}

\subsection{Exact Kernel Recursion and Conditional Moments}\label{sec:recursion}

Defining the empirical kernel $G_{ab}^\ell := n^{-1}\sum_{i=1}^n \phi_i^\ell(a)\phi_i^\ell(b)$ and
substituting the update rule $\phi_i^{\ell+1} = \phi_i^\ell + \eps\,\eta_i^\ell$ gives the exact one-step update:
\begin{equation}
  G_{ab}^{\ell+1} = G_{ab}^\ell + \eps\, H_{ab}^\ell + \eps^2\, J_{ab}^\ell,
  \label{eq:kernel-update}
\end{equation}
\begin{align}
  H_{ab}^\ell &:= \frac{1}{n}\sum_{i=1}^n
    \bigl[\phi_i^\ell(a)\eta_i^\ell(b) + \eta_i^\ell(a)\phi_i^\ell(b)\bigr],
  \label{eq:H-def}\\
  J_{ab}^\ell &:= \frac{1}{n}\sum_{i=1}^n \eta_i^\ell(a)\eta_i^\ell(b).
  \label{eq:J-def}
\end{align}

Introducing the natural filtration $\cF_\ell := \sigma(\phi^0, \phi^1, \dots, \phi^\ell)$, we compute the conditional moments.
Since $\E[\eta_i^\ell(a)|\cF_\ell] = 0$, we have $\E[H_{ab}^\ell|\cF_\ell] = 0$.
Using the independence of different neurons and Wick's theorem, the conditional covariance of $H^\ell$ is:
\begin{equation}
  n\,\mathrm{Cov}(H_{ab}^\ell,\, H_{cd}^\ell \mid \cF_\ell)
  = G_{ac}^\ell\hatQ_{bd}^\ell + G_{ad}^\ell\hatQ_{bc}^\ell
    + G_{bc}^\ell\hatQ_{ad}^\ell + G_{bd}^\ell\hatQ_{ac}^\ell,
  \label{eq:H-cov}
\end{equation}
where $\hatQ_{ab}^\ell \equiv \cG_{\eta,ab}^\ell[\phi^\ell] = \Cb + \frac{\Cw}{n}\sum_j \sigma(\phi_j^\ell(a))\sigma(\phi_j^\ell(b))$ is the exact conditional drift kernel.
For $J^\ell$, we have $\E[J_{ab}^\ell|\cF_\ell] = \hatQ_{ab}^\ell$, which yields the exact conditional drift:
\begin{equation}
  \E[G^{\ell+1} - G^\ell \mid \cF_\ell] = \eps^2\,\hatQ^\ell.
  \label{eq:exact-drift}
\end{equation}
A Gaussian Wick expansion gives $n\,\mathrm{Cov}(J_{ab}^\ell,\, J_{cd}^\ell \mid \cF_\ell) = \hatQ_{ac}^\ell\hatQ_{bd}^\ell + \hatQ_{ad}^\ell\hatQ_{bc}^\ell$,
implying $r_J^\ell := \sqrt{n}(J^\ell - \hatQ^\ell) = O_p(1)$ and guaranteeing that $\eps^2 r_J^\ell$ is subleading in the fluctuation recursion.

Defining the full mean $\bar{K}^\ell := \E[G^\ell]$ and ensemble-averaging (\ref{eq:exact-drift}) gives
\begin{equation}
  \bar{K}^{\ell+1} - \bar{K}^\ell = \eps^2\,\E[\hatQ^\ell] = \eps^2\bigl(\Cb + \Cw\bar{S}^\ell\bigr),
  \label{eq:mean-sigma}
\end{equation}
where $S_{ab}^\ell := n^{-1}\sum_i\sigma(\phi_i^\ell(a))\sigma(\phi_i^\ell(b))$ is the sigma-kernel and $\bar{S}^\ell := \E[S^\ell]$.

Similarly, defining the centered fluctuation $v^\ell := \sqrt{n}(G^\ell - \bar{K}^\ell)$ and the kernel fluctuation covariance $V_4^\ell := \E[v^\ell(v^\ell)^\top]$,
we have the exact fluctuation update:
\begin{equation}
  v^{\ell+1} - v^\ell
  = \eps\,\zeta^\ell + \eps^2 r_J^\ell + \eps^2 d^\ell,
  \label{eq:v-exact}
\end{equation}
with $\zeta^\ell := \sqrt{n}\,H^\ell$ and $d^\ell := \sqrt{n}(\hatQ^\ell - \E[\hatQ^\ell])$.

\paragraph{Gaussian reference map and derivatives.}
To define the Gaussian reference background and the subsequent closures, we introduce the following operators for any symmetric positive-definite matrix $K$:
\begin{align}
  E_2(K)_{ab} &:= \E_{z\sim\mathcal{N}(0,K)}[\sigma(z_a)\sigma(z_b)],
  \label{eq:E2-def}\\
  Q(K)_{ab} &:= \Cb + \Cw\,E_2(K)_{ab},
  \label{eq:Q-def}\\
  (\chi_K[v])_{ab} &:= \frac{\partial Q_{ab}}{\partial K_{cd}}\bigg|_K v_{cd},
  \label{eq:chi-def}\\
  (\Hess[K]:V)_{ab} &:= \frac{\partial^2 Q_{ab}}
  {\partial K_{cd}\partial K_{ef}}\bigg|_K V_{cd,ef}.
  \label{eq:hess-def}
\end{align}
Here $\chi_K$ and $\Hess[K]$ are Fréchet derivatives on $\mathrm{Sym}_N$, with the partial derivative using the symmetrized convention that identifies $(c,d)$ with $(d,c)$.
The factor $1/2$ multiplying $\Hess[K]:V$ below is the standard second-order Taylor coefficient.
For the derivative-based closures, we assume that $Q$ is twice Fréchet differentiable near the positive-definite background trajectory and that the moments required for the expansion are controlled;
the exact identities require only the relevant activation moments and therefore also cover nonsmooth activations such as ReLU.

Defining
$\Delta_j^K(a,b):=\sigma(\phi_j(a))\sigma(\phi_j(b))-E_2(K)_{ab}$ gives
$\cG_\eta[\phi]=Q(K)+(\Cw/n)\sum_j\Delta_j^K$.
Under the Gaussian reference measure, $\E[\Delta_j^K]=0$.
We now define $K_0^\ell$ by
$K_0^{\ell+1}-K_0^\ell=\eps^2Q(K_0^\ell)$ and set
\begin{equation}
  \Kmic^\ell:=n(\bar K^\ell-K_0^\ell),\qquad
  \Uex^\ell:=n\bigl(\bar S^\ell-E_2(K_0^\ell)\bigr).
  \label{eq:exact-mean-source-definitions}
\end{equation}
Subtracting the reference update from the exact mean identity (\ref{eq:mean-sigma}) yields the exact finite-$n$ recursion
\begin{equation}
  \Kmic^{\ell+1}-\Kmic^\ell=\eps^2\Cw\,\Uex^\ell.
  \label{eq:K1mic-recursion}
\end{equation}
Thus $\Uex^\ell$ is the exact source of $\Kmic$ around the chosen Gaussian background.

\subsection{A Hierarchy of $G$-only Closures}\label{sec:closure}

All statements up to this point are exact identities or definitions.
We next separate an exact finite-step covariance identity, closure assumptions, and the continuous-depth limit.
Define the centered sigma-kernel fluctuation
\begin{equation}
  \nu^\ell:=\sqrt{n}\bigl(S^\ell-\E[S^\ell]\bigr),
  \qquad d^\ell=\Cw\nu^\ell ,
  \label{eq:d-nu-exact}
\end{equation}
where the second equality follows exactly at finite $n$ from $\hatQ^\ell=\Cb+\Cw S^\ell$.
Let $C_{\nu v}^\ell:=\E[\nu^\ell(v^\ell)^\top]$ and
\begin{equation}
  \Sigma_{\mathrm{mic}}^\ell
  :=\E\left[n\,\mathrm{Cov}(H^\ell,H^\ell\mid\cF_\ell)\right].
  \label{eq:Sigma-mic-def}
\end{equation}
In expanding (\ref{eq:v-exact}), the cross terms with $v^\ell$ and either $\zeta^\ell$ or $r_J^\ell$ vanish by conditional centering.
The $\zeta^\ell$--$r_J^\ell$ term is a conditional third moment of the centered Gaussian increment and therefore vanishes, while the $\zeta^\ell$--$d^\ell$ and $r_J^\ell$--$d^\ell$ terms vanish because $d^\ell$ is $\cF_\ell$-measurable and the other factor is conditionally centered.
The remaining terms give the exact finite-$\eps$ identity
\begin{equation}
  \frac{V_4^{\ell+1}-V_4^\ell}{\eps^2}
  =\Sigma_{\mathrm{mic}}^\ell
   +\Cw\left(C_{\nu v}^\ell+(C_{\nu v}^\ell)^\top\right)
   +R_\eps^\ell ,
  \qquad
  R_\eps^\ell:=\eps^2\E[(r_J^\ell+d^\ell)(r_J^\ell+d^\ell)^\top].
  \label{eq:V4-exact-discrete}
\end{equation}
Its leading continuous-depth limit omits $R_\eps^\ell=O(\eps^2)$, but remains unclosed because it contains sigma-kernel statistics.

We define the exact sigma-kernel defect left unresolved by a state description in $G^\ell$ alone by
\begin{equation}
  \Delta_Q^\ell:=\hatQ^\ell-Q(G^\ell).
  \label{eq:DeltaQ-def}
\end{equation}
The closures below are not independent Gaussian assumptions.
They reduce distinct effects of the unresolved sigma-kernel sector: the leading sample-wise defect $\Delta_Q^\ell$, the covariance mismatch $\Delta_{\mathrm{proj}}^\ell$, and the NLO mean bias $B_Q^\ell$.

\paragraph{(GC0) Leading $G$-state reduction.}
At leading order, following~\cite{roberts2022principles}, the $G$-only description supplements propagation of chaos~\cite{sznitman1991topics} with the sample-wise reduction
\begin{equation}
  \Delta_Q^\ell=\hatQ^\ell-Q(G^\ell)=o_p(1)
  \quad (n \to \infty).
  \label{eq:GC0}
\end{equation}
For the $K_0$ equation itself, only the weaker consequence $\E[\Delta_Q^\ell]=o(1)$ is required; together with concentration of $G^\ell$ around $K_0^\ell$, it gives $\E[\hatQ^\ell]=Q(K_0^\ell)+o(1)$.
The stronger sample-wise form in (\ref{eq:GC0}) is used later to define a local Markov reference action in $G$ alone.

\paragraph{(PROJ) Projected fluctuation closure.}
The exact leading covariance equation contains two sigma-kernel-dependent structures: the microscopic noise source $\Sigma_{\mathrm{mic}}^\ell$ and the cross-covariance between the sigma-kernel and kernel fluctuations, $C_{\nu v}^\ell:=\E[\nu^\ell(v^\ell)^\top]$.
We therefore define the projection defect exactly by
\begin{equation}
  \Delta_{\mathrm{proj}}^\ell
  :=\Cw C_{\nu v}^\ell-\chi_{K_0^\ell}V_4^\ell .
  \label{eq:projection-defect}
\end{equation}
For the smooth activation sector considered in the closure analysis, an i.i.d.\ Gaussian initial ensemble, Gaussian integration by parts, and the Wishart covariance of $G^0$ give
\begin{equation}
  \Cw C_{\nu v}^0=\chi_{K_0^0}V_4^0,
  \qquad\text{and hence}\qquad
  \Delta_{\mathrm{proj}}^0=0
  \label{eq:PROJ-init}
\end{equation}
exactly at finite $n$;
Appendix~\ref{sec:appx-proj} gives the componentwise derivation.
Motivated by this exact initial identity and the small update produced by each residual block, the minimal $G$-only Gaussian-reference baseline assumes that the projection defect remains negligible over the depth range of interest:
\begin{equation}
  \Delta_{\mathrm{proj}}^\ell\simeq0 .
  \label{eq:PROJ}
\end{equation}
This is a covariance-level closure, not a sample-wise approximation: it does not require the condition
$\hatQ^\ell-Q(G^\ell)=o_p(n^{-1/2})$, which does not hold in general.
Its accuracy at later depth is tested directly by measuring $\Delta_{\mathrm{proj}}^\ell$; it is not a consequence of (GC0).

At the same leading order, evaluating the exact conditional noise covariance in~(\ref{eq:H-cov}) on the deterministic $G$-only background defines
\begin{equation}
  \Sigma_{ab,cd}(K)
  :=
  K_{ac}Q_{bd}(K)+K_{ad}Q_{bc}(K)
  +K_{bc}Q_{ad}(K)+K_{bd}Q_{ac}(K).
  \label{eq:Sigma-def}
\end{equation}
Under concentration of $G^\ell$ and $\hatQ^\ell$ and suitable moment bounds, the corresponding source-background evaluation is
\begin{equation}
  \Sigma_{\mathrm{mic}}^\ell=\Sigma(K_0^\ell)+o(1).
  \label{eq:source-background}
\end{equation}
This is a leading evaluation of the exact source, rather than an independent sigma-kernel projection assumption.

\paragraph{(GC1) NLO $G$-only bias closure.}
At finite width, define the conditional-bias defect
\begin{equation}
  B_{Q,n}^\ell:=n\,\E[\Delta_Q^\ell]
  =n\,\E[\hatQ^\ell-Q(G^\ell)].
  \label{eq:BQ-def}
\end{equation}
Equivalently,
$B_{Q,n}^\ell=n\Cw\,\E[\E[S^\ell\mid G^\ell]-E_2(G^\ell)]$,
so it measures the conditional-on-$G$ sigma-kernel bias not represented by the Gaussian function $E_2(G^\ell)$.
When taking the NLO large-width limit below, we assume that
$B_{Q,n}^\ell$ admits an $O(1)$ limit, denoted by $B_Q^\ell$.
Assuming sufficient smoothness and moment control, the $Q(G)$ part has the second-order delta expansion
\begin{equation}
  \E[Q(G^\ell)]
  = Q(\bar{K}^\ell) + \frac{1}{2n}\Hess[K_0^\ell]:V_4^\ell + o(n^{-1}).
  \label{eq:Q-delta-expansion}
\end{equation}
Here the Hessian may be evaluated at $K_0^\ell$ at the retained order because $\bar K^\ell=K_0^\ell+O(n^{-1})$.
Consequently,
\begin{equation}
  \E[\hatQ^\ell]
  =Q(\bar K^\ell)+\frac{1}{2n}\Hess[K_0^\ell]:V_4^\ell
   +\frac{1}{n}B_{Q,n}^\ell+o(n^{-1}).
  \label{eq:EQ-with-BQ}
\end{equation}
The Hessian term is therefore a Taylor correction, not an additional sigma-kernel closure.
The $G$-only NLO baseline makes the closure
\begin{equation}
  B_Q^\ell\simeq0 .
  \label{eq:GC1}
\end{equation}

The three closure levels discard distinct information in the $G$-only reduction.
(GC0) neglects the leading sample-wise defect $\Delta_Q^\ell$, (PROJ) neglects the covariance projection defect $\Delta_{\mathrm{proj}}^\ell$, and (GC1) neglects the limiting conditional-bias term $B_Q^\ell$.

The logical status of the exact identities, closure assumptions, and diffusion approximation is summarized in Table~\ref{tab:exact-assumed}.

\begin{table}[h]
  \centering
  \caption{Logical status of each derivation step}
  \label{tab:exact-assumed}
  \small
  \begin{tabular}{p{5.8cm}p{4.9cm}l}
    \hline
    Statement & Status & Section \\
    \hline
    Increment $\eta^\ell$ is conditionally Gaussian & Exact (finite $n$) & \S\ref{sec:block} \\
    Empirical kernel mean/covariance recursion & Exact identity & \S\ref{sec:recursion} \\
    (GC0): $\Delta_Q^\ell=o_p(1)$ & Leading $G$-state reduction & \S\ref{sec:closure} \\
    Exact $V_4$ increment (\ref{eq:V4-exact-discrete}) & Exact (finite $n,\eps$) & \S\ref{sec:closure} \\
    (PROJ): $\Delta_{\mathrm{proj}}^\ell\simeq0$ & Covariance-level closure & \S\ref{sec:closure} \\
    $\Sigma_{\mathrm{mic}}^\ell=\Sigma(K_0^\ell)+o(1)$ & Background evaluation under concentration & \S\ref{sec:closure} \\
    Expansion (\ref{eq:Q-delta-expansion}) & Delta expansion under smoothness/moment bounds & \S\ref{sec:closure} \\
    (GC1): $B_Q^\ell\simeq0$ & NLO $G$-only mean closure & \S\ref{sec:closure} \\
    $K_0$ ODE (\ref{eq:K0-ode}) & Derived under (GC0) only & \S\ref{sec:closure-odes} \\
    $V_4$ ODE (\ref{eq:V4-ode}) & (PROJ), source background evaluation, and $\eps\to0$ & \S\ref{sec:closure-odes} \\
    $\Keft$ ODE (\ref{eq:K1-ode}) & Delta expansion and (GC1) & \S\ref{sec:closure-odes} \\
    Local $G$-only collective SDE/action (\S\ref{sec:eft})
    & Local Markov reference model implementing the $G$-only baseline,
      followed by the Kramers--Moyal diffusion approximation
    & \S\ref{sec:eft} \\
    Scalar Hermite--Galerkin system & Finite-$R$ closure ($N=1$, $\tanh$; leading fluctuation/NLO mean) & \S\ref{sec:hermite} \\
    \hline
  \end{tabular}
\end{table}

\subsection{Continuous-Depth ODEs for $K_0$, $V_4$, and $\Keft$}
\label{sec:closure-odes}

Setting $\bar{K}^\ell = K_0^\ell + O(n^{-1})$,
(GC0) gives $\E[\hatQ^\ell] = Q(K_0^\ell) + o(1)$.
Extracting the $O(1)$ term from the mean recursion (\ref{eq:mean-sigma}) gives the discrete recursion
$K_0^{\ell+1} - K_0^\ell = \eps^2\, Q(K_0^\ell)$.
Taking the continuous limit $t = \eps^2\ell$:
\begin{equation}
  \partial_t K_0 = Q(K_0).
  \label{eq:K0-ode}
\end{equation}

For the fluctuation covariance $V_4$, applying (PROJ) and the leading source background evaluation to the exact identity (\ref{eq:V4-exact-discrete}) and then taking $\eps\to0$ gives the $G$-only baseline with the source (\ref{eq:Sigma-def}).
The resulting continuous-depth equation is
\begin{equation}
  \partial_t V_4 = \chi_{K_0}V_4 + V_4\chi_{K_0}^\top + \Sigma(K_0).
  \label{eq:V4-ode}
\end{equation}

To derive the defect-augmented NLO equation, we parametrize the deviation of
$\bar{K}^\ell$ from $K_0^\ell$ by the asymptotic ansatz
$\bar{K}^\ell = K_0^\ell + \frac{1}{n}K_1^\ell + o(n^{-1})$,
where $K_1^\ell=O(1)$ is the large-width limit of
$\Kmic^\ell=n(\bar{K}^\ell-K_0^\ell)$.
Expanding (\ref{eq:EQ-with-BQ}) around $K_0^\ell$ and using
$B_{Q,n}^\ell\to B_Q^\ell$ gives the defect-augmented NLO equation
\begin{equation}
  \E[\hatQ^\ell]
  = Q(K_0^\ell) + \frac{1}{n}\chi_{K_0^\ell}[K_1^\ell]
    + \frac{1}{2n}\Hess[K_0^\ell]:V_4^\ell
    +\frac{1}{n}B_Q^\ell+o(n^{-1}).
  \label{eq:EQ-expand}
\end{equation}
Extracting the $O(1/n)$ term and taking $t=\eps^2\ell$ yields
\begin{equation}
  \partial_t K_1
  =\chi_{K_0}[K_1]+\frac12\Hess[K_0]:V_4+B_Q .
  \label{eq:K1-defect}
\end{equation}

At leading order, the bridge identity~(\ref{eq:VG-Vphi-bridge}) in Appendix~\ref{sec:appx-banta} decomposes the kernel covariance $V_4$ into a single-neuron Gaussian contraction and the connected preactivation four-point sector.
We denote the former by
\begin{equation}
  \Omega(K)_{ab,cd}
  :=K_{ac}K_{bd}+K_{ad}K_{bc}.
  \label{eq:Omega-main}
\end{equation}
With this contraction separated from $V_4$, equation~(\ref{eq:K1-defect}) becomes
\begin{equation}
  \partial_t K_1
  =\chi_{K_0}[K_1]
   +\frac12\Hess[K_0]:\left[V_4-\Omega(K_0)\right]
   +B_Q^{\mathrm{res}} .
  \label{eq:K1-Wishart-subtracted}
\end{equation}
Here
\begin{equation}
  B_Q^{\mathrm{res}}
  :=B_Q+\frac12\Hess[K_0]:\Omega(K_0).
  \label{eq:BQ-Wishart-limit}
\end{equation}
Equivalently, the leading conditional-Wishart bias and the corresponding finite-width remainder are
\begin{equation}
  B_Q^{\mathrm W}(K_0)
  :=-\frac12\Hess[K_0]:\Omega(K_0),
  \qquad
  B_{Q,n}^{\mathrm{res}}
  :=B_{Q,n}-B_Q^{\mathrm W}(K_0).
  \label{eq:BQ-Wishart-finite}
\end{equation}
When the finite-width remainder has a large-width limit, it equals $B_Q^{\mathrm{res}}$ in~(\ref{eq:BQ-Wishart-limit}).
Setting $B_Q^{\mathrm{res}}=0$ defines the Wishart-subtracted candidate closure.
The defect-augmented equation~(\ref{eq:K1-defect}) is not predictive unless $B_Q$ is closed.
Setting $B_Q=0$ defines the reference $G$-only GC1 equation:
\begin{equation}
  \partial_t \Keft
  = \chi_{K_0}[\Keft] + \frac{1}{2}\Hess[K_0]:V_4.
  \label{eq:K1-ode}
\end{equation}

\begin{remark}
  The three ODEs require different levels of approximation.
  (\ref{eq:K0-ode}) requires only (GC0) (at $o_p(1)$ precision) and
  concentration of $G^\ell$ around $K_0^\ell$.
  (\ref{eq:V4-ode}) additionally requires the covariance projection (PROJ) and
  leading background evaluation of the noise source, rather than a sample-wise
  identification of $\hatQ$ with $Q(G)$.
  (\ref{eq:K1-ode}) further requires the second-order delta expansion and
  (GC1), namely $B_Q\simeq0$.
  In particular, (\ref{eq:K1-ode}) is not the exact finite-width recursion for
  $\Kmic$:
  the exact source is $\Uex^\ell$ in (\ref{eq:K1mic-recursion}), which in general differs from
  the EFT source $\chi_{K_0}[\Keft] + \frac{1}{2}\Hess[K_0]:V_4$.
\end{remark}

\subsection{Direct Decomposition of the Closure Defects}\label{sec:defect-decomp}

Having defined the $G$-only covariance equation (\ref{eq:V4-ode}), we now make its finite-step closure errors directly measurable.
Define
\begin{align}
  \delta^\ell&:=\sqrt n\,[\hatQ^\ell-Q(G^\ell)],\notag\\
  \tau^\ell&:=\sqrt n\left(Q(G^\ell)-\E[Q(G^\ell)]\right)
              -\chi_{K_0^\ell}v^\ell .
  \label{eq:delta-tau}
\end{align}
Using $d^\ell=\Cw\nu^\ell$ and the definitions above gives the exact centered identity
\begin{equation}
  \Cw\nu^\ell-\chi_{K_0^\ell}v^\ell
  =\delta^\ell-\E[\delta^\ell]+\tau^\ell .
  \label{eq:delta-tau-centered}
\end{equation}
The exact projection defect is therefore
\begin{equation}
  \Delta_{\mathrm{proj}}^\ell
  =\mathrm{Cov}(\delta^\ell,v^\ell)
   +\mathrm{Cov}(\tau^\ell,v^\ell).
  \label{eq:projection-defect-diagnostic}
\end{equation}
Define the finite-step residual of the $G$-only covariance update by
\begin{align}
  R_{V_4}^\ell
  :={}&
  \frac{V_4^{\ell+1}-V_4^\ell}{\eps^2}
  -\chi_{K_0^\ell}V_4^\ell
  -V_4^\ell\chi_{K_0^\ell}^{\top}
  -\Sigma(K_0^\ell).
  \label{eq:RV4-def}
\end{align}
Substitution of the exact finite-step identity (\ref{eq:V4-exact-discrete}) gives
\begin{equation}
  R_{V_4}^\ell
  =
  R_\delta^\ell+R_\tau^\ell
  +\left[\Sigma_{\mathrm{mic}}^\ell-\Sigma(K_0^\ell)\right]
  +R_\eps^\ell ,
  \label{eq:V4-defect-decomposition}
\end{equation}
where
\begin{equation}
\begin{aligned}
  R_\delta^\ell
  &:=
  \mathrm{Cov}(\delta^\ell,v^\ell)
  +\mathrm{Cov}(v^\ell,\delta^\ell),\\
  R_\tau^\ell
  &:=
  \mathrm{Cov}(\tau^\ell,v^\ell)
  +\mathrm{Cov}(v^\ell,\tau^\ell).
\end{aligned}
\label{eq:Rdelta-Rtau}
\end{equation}
Thus the symmetrized projection defect is exactly
$\Delta_{\mathrm{proj}}^\ell+(\Delta_{\mathrm{proj}}^\ell)^\top
=R_\delta^\ell+R_\tau^\ell$; the other two diagnostic contributions are
$\Sigma_{\mathrm{mic}}^\ell-\Sigma(K_0^\ell)$ and $R_\eps^\ell$.

The SDE and collective action in Section~\ref{sec:eft} introduce one additional layer of approximation beyond the closure assumptions above:
the Kramers--Moyal diffusion limit ($\eps\to 0$), which is not an exact rewriting of the discrete recursion.
All exact statements from Sections~\ref{sec:block}--\ref{sec:kernel} remain valid as stated.

\section{Collective Bilocal Stochastic EFT}\label{sec:eft}

\subsection{Collective Bilocal SDE and Effective Action}

This section gives a diagrammatic interpretation of the local $G$-only reference equations from Section~\ref{sec:kernel}; it does not introduce an independent closure.
At the level of the local $G$-only reference model, we assign the one-step cumulants
\begin{align}
  \E[\Delta G^\ell\mid G^\ell]
  &=\eps^2Q(G^\ell)+o(\eps^2),\notag\\
  \mathrm{Cov}(\Delta G^\ell\mid G^\ell)
  &=\frac{\eps^2}{n}\Sigma(G^\ell)+o(\eps^2/n).
  \label{eq:reference-cumulants}
\end{align}
The first relation uses (GC0) for the leading drift, while the second uses the source-background evaluation.
Their fluctuation expansion implements the projected $G$-only baseline rather than the complete sigma-kernel dynamics.
The leading stochastic term $\eps H^\ell$ is exactly Gaussian conditional on $\cF_\ell$, and the subleading fluctuation $\eps^2 r_J^\ell$ is guaranteed to be $O_p(\eps^2)$ since $r_J^\ell = O_p(1)$.
A Kramers--Moyal diffusion approximation~\cite{risken1989fokker,gardiner2004handbook} yields an effective Markov kernel, leading to the continuous-limit bilocal collective SDE:
\begin{equation}
  dG_{ab} = Q_{ab}(G)\,dt + \frac{1}{\sqrt{n}}B_{ab,\mu}(G)\,dW_\mu,
  \qquad B(G)B(G)^\top = \Sigma(G).
  \label{eq:collective-sde}
\end{equation}
Introducing the response field $\hat{G}_{ab}$, the approximate collective MSRJD action for the path probability is
\begin{equation}
  S_{\mathrm{eff}}[G,\hat{G}]
  = \int_0^T dt\;\left[\hat{G}_{ab}\bigl(\partial_t G_{ab} - Q_{ab}(G)\bigr)
    - \frac{1}{2n}\hat{G}_{ab}\,\Sigma_{ab,cd}(G)\,\hat{G}_{cd}\right]
  \label{eq:collective-action}
\end{equation}
(after the contour rotation $i\hat{G} \mapsto \hat{G}$).
We interpret (\ref{eq:collective-action}) as a local reference action for the $G$-only baseline.
Its diagrams generate the baseline transport, noise, and tadpole terms, but it is not the complete reduced action obtained by integrating out sigma-kernel fluctuations; such an elimination would generally induce additional nonlocal or colored-noise terms.
In the discrete theory, $\hat{G}_\ell$ is associated with the equation residual at the next step (pre-point convention), meaning $\langle \hat{G}_\ell \cdot G_\ell \rangle = 0$.
This discrete causality carries over to the continuous limit as the It\^o convention, uniquely determining the equal-time retarded response function:
\begin{equation}
  R(s,s) = 0.
  \label{eq:ito-causal}
\end{equation}

\subsection{Expansion around the Background $K_0$}\label{sec:background}

Write $G_{ab} = K_{0,ab} + n^{-1/2}g_{ab}$ and $\hat{G}_{ab} = n^{1/2}\hat{g}_{ab}$.
The background equation $\partial_t K_0 = Q(K_0)$ and a Taylor expansion give
\begin{equation}
  \partial_t G_{ab} - Q_{ab}(G)
  = \frac{1}{\sqrt{n}}\bigl(\partial_t g_{ab} - \chi_{ab,cd}(K_0)g_{cd}\bigr)
    - \frac{1}{2n}\Hess_{ab;cd,ef}(K_0)\,g_{cd}\,g_{ef} + O(n^{-3/2}).
\end{equation}
With $M_{ab,cd;ef} := \partial\Sigma_{ab,cd}/\partial K_{ef}|_{K_0}$:
\begin{equation}
  \Sigma_{ab,cd}(G) = \Sigma_{ab,cd}(K_0)
    + \frac{1}{\sqrt{n}}M_{ab,cd;ef}(K_0)\,g_{ef} + O(n^{-1}).
  \label{eq:sigma-expand}
\end{equation}
The noise cubic $S^{(3)}_{\mathrm{noise}}$ arises from this $O(n^{-1/2})$ term.
Substituting into the full action:
\begin{equation}
  S_{\mathrm{eff}} = S^{(2)}[g,\hat{g}] - \frac{1}{\sqrt{n}}S^{(3)}_{\mathrm{int}}[g,\hat{g}] + O(n^{-1}),
  \label{eq:action-expand}
\end{equation}
\begin{equation}
  S^{(2)}[g,\hat{g}]
  = \int dt\;\left[\hat{g}_{ab}\bigl(\partial_t g_{ab} - \chi_{ab,cd}(K_0)\,g_{cd}\bigr)
    - \frac{1}{2}\hat{g}_{ab}\,\Sigma_{ab,cd}(K_0)\,\hat{g}_{cd}\right],
  \label{eq:quadratic-action}
\end{equation}
\begin{align}
  S^{(3)}_{\mathrm{int}} &:= S^{(3)}_{\mathrm{drift}} + S^{(3)}_{\mathrm{noise}}, \notag\\
  S^{(3)}_{\mathrm{drift}} &:= \frac{1}{2}\int dt\;\hat{g}_{ab}\,\Hess_{ab;cd,ef}(K_0)\,g_{cd}\,g_{ef},
  \label{eq:S3-drift}\\
  S^{(3)}_{\mathrm{noise}} &:= \frac{1}{2}\int dt\;\hat{g}_{ab}\,M_{ab,cd;ef}(K_0)\,g_{ef}\,\hat{g}_{cd}.
  \label{eq:S3-noise}
\end{align}
Power counting gives $O(1)$ for the free theory, $O(n^{-1/2})$ for one cubic vertex insertion, and $O(n^{-1})$ for two insertions.

The next subsection introduces the free propagators $R$ and $C$,
and shows that the ODEs for $V_4$ and $\Keft$ admit a diagrammatic reinterpretation from this cubic expansion.

\subsection{Diagrammatic Interpretation}\label{sec:diag}

Using the pair index notation $A=(ab)$, $B=(cd)$, the nonzero free propagators in the quadratic action (\ref{eq:quadratic-action}) are the retarded response $R_{AB}(t,s) := \langle g_A(t)\hat{g}_B(s)\rangle_0$ (nonzero only for $t > s$) and the correlation $C_{AB}(t,s) := \langle g_A(t)g_B(s)\rangle_0$.
The It\^o convention (\ref{eq:ito-causal}) enforces $R(s,s) = 0$.
The equal-time covariance $V_{4;AB}(t) := C_{AB}(t,t)$ corresponds to $V_4$ from Section~\ref{sec:kernel}.
Table~\ref{tab:feynman-rules} summarizes the propagators and vertices used below, as read off from the action (\ref{eq:quadratic-action})--(\ref{eq:S3-noise}).

\begin{table}[htbp]
\centering
\caption{Feynman rules for the quadratic and cubic terms of the collective bilocal EFT.
Time flows from left to right.
The arrowed line denotes the response propagator $R$, and the plain line denotes the covariance $C$.}
\label{tab:feynman-rules}
\small
\renewcommand{\arraystretch}{1.3}
\begin{tabular}{@{}cl>{\raggedright\arraybackslash}p{7.5cm}@{}}
  \hline
  Symbol & Element & Factor \\
  \hline
  \multicolumn{3}{@{}l@{}}{\textit{Propagators}} \\
  $\vcenter{\hbox{%
    \begin{tikzpicture}[baseline=-0.5ex,x=1em,y=1em]
      \draw[resp](0,0)--(1.8,0);
    \end{tikzpicture}}}$ &
  Retarded response $R$ &
  $R_{AB}(t,s)$; arrow from the $\hat g$ leg at time $s$ to the $g$ leg at time $t$; retarded, with $R(s,s)=0$ \\
  $\vcenter{\hbox{%
    \begin{tikzpicture}[baseline=-0.5ex,x=1em,y=1em]
      \draw[line width=.45pt](0,0)--(1.8,0);
    \end{tikzpicture}}}$ &
  Correlation $C$ &
  $C_{AB}(t,s)=\langle g_A(t)g_B(s)\rangle_0$;
undirected $g$-$g$ contraction, with $C(s,s)=V_4(s)$ \\
  \hline
  \multicolumn{3}{@{}l@{}}{\textit{Vertices}} \\
  $\vcenter{\hbox{%
    \begin{tikzpicture}[baseline=-0.5ex,x=1em,y=1em]
      \node[bulk](S) at (0,0){};
      \draw[resp](S.east)--(1.6, 0.55);
      \draw[resp](S.east)--(1.6,-0.55);
    \end{tikzpicture}}}$ &
  $\hat{g}\hat{g}$ noise-source vertex $\Sigma$ &
  $\Sigma_{AB}(K_0)$; local insertion with two $\hat g$ legs \\
  $\vcenter{\hbox{%
    \begin{tikzpicture}[baseline=-0.5ex,x=1em,y=1em]
      \node[dvert](V) at (0,0){};
      \draw[resp](V.east)--(1.2,0);
      \draw[line width=.45pt](V.west)--(-1.0, 0.4);
      \draw[line width=.45pt](V.west)--(-1.0,-0.4);
    \end{tikzpicture}}}$ &
  $\hat{g}gg$ drift cubic vertex $\tfrac{1}{2}\Hess$ &
  $\tfrac12\,\Hess_{A;BC}(K_0)$; one $\hat g$ leg and two $g$ legs \\
  $\vcenter{\hbox{%
    \begin{tikzpicture}[baseline=-0.5ex,x=1em,y=1em]
      \node[dvert](V) at (0,0){};
      \draw[resp](V.north east)--(0.9, 0.5);
      \draw[resp](V.south east)--(0.9,-0.5);
      \draw[line width=.45pt](V.west)--(-1.0,0);
    \end{tikzpicture}}}$ &
  $\hat{g}\hat{g}g$ noise cubic vertex $\tfrac{1}{2}M$ &
  $\tfrac12\,M_{AB;C}(K_0)$; two $\hat g$ legs and one $g$ leg \\
  \hline
  \multicolumn{3}{@{}l@{}}{\textit{External}} \\
  $\vcenter{\hbox{%
    \begin{tikzpicture}[baseline=-0.5ex,x=1em,y=1em]
      \node[obs] at (0,0){};
    \end{tikzpicture}}}$ &
  Observation point &
  External $g$ insertion at the observation time \\
  \hline
\end{tabular}
\end{table}

\paragraph{Free theory and $V_4$ transport.}
Applying the It\^o formula to $g_Ag_B$ in the free theory SDE $dg = \chi_{K_0}g\,dt + B(K_0)\,dW$ reproduces the $V_4$ ODE (\ref{eq:V4-ode}).
Its integral form admits a direct diagrammatic interpretation:
\begin{equation}
  V_{4;AB}(t)
  = \underbrace{R_{AC}(t,0)\,V_{4;CD}(0)\,R_{BD}(t,0)}_{\diagVFourTransport}
  + \int_0^t du\;
    \underbrace{R_{AC}(t,u)\,\Sigma_{CD}(K_0(u))\,R_{BD}(t,u)}_{\diagVFourSource}.
  \label{eq:V4-integral}
\end{equation}
The first term transports the initial covariance $V_4(0)$ (the blob labeled $V_4^0$) forward along two response lines; the second integrates the noise source $\Sigma(K_0(u))$ at each intermediate time $u$, propagated to the observation point by two response lines.

\paragraph{Cubic vertices and $\Keft$ tadpole.}

From the power counting in Section~\ref{sec:background}, the mean of the fluctuation field $m_A(t) := \langle g_A(t)\rangle$ is $O(n^{-1/2})$, meaning
$m_A(t) = n^{-1/2}\Keft{}_A(t) + o(n^{-1/2})$.
In the free theory $\langle g\rangle_0 = 0$, so the first nonzero contribution comes from a single cubic vertex insertion:
\begin{equation}
  \langle g_A(t)\rangle
  = \frac{1}{\sqrt{n}}\langle g_A(t)\,S^{(3)}_{\mathrm{int}}\rangle_0
    +o(n^{-1/2}).
  \label{eq:1pt-expand}
\end{equation}

For the drift vertex $S^{(3)}_{\mathrm{drift}}$, the relevant Wick contraction is
\begin{equation}
  \langle g_A(t)\hat{g}_B(s)g_C(s)g_D(s)\rangle_0
  =R_{AB}(t,s)\,V_{4;CD}(s),
\end{equation}
because the equal-time response $R(s,s)$ vanishes.
This gives the one-loop tadpole integral:
\begin{equation}
  \langle g_A(t)\,S^{(3)}_{\mathrm{drift}}\rangle_0
  = \frac{1}{2}\int_0^t ds\;R_{AB}(t,s)\,\Hess_{B;CD}(s)\,V_{4;CD}(s).
  \label{eq:drift-tadpole}
\end{equation}
Adding the homogeneous solution $R_{AB}(t,0)\,\Keft{}_B(0)$ yields the complete integral equation for $\Keft$:
\begin{equation}
  \Keft{}_{A}(t)
  = \underbrace{R_{AB}(t,0)\,\Keft{}_{B}(0)}_{\diagKOneTransport}
  + \int_0^t ds\;
    \underbrace{\tfrac{1}{2}\,R_{AB}(t,s)\,\Hess_{B;CD}(s)\,V_{4;CD}(s)}_{\diagKOneSource}.
  \label{eq:Keft-integral}
\end{equation}
The first term propagates the initial condition $\Keft(0)$ (the blob labeled $K_1^0$) forward via a single response line; the tadpole term integrates the one-loop insertion $\frac{1}{2}\Hess[K_0(s)]:V_4(s)$, in which the two $g$ legs of the drift vertex are contracted by the covariance propagator $C$ to form a self-loop.
Applying $(\partial_t - \chi_{K_0(t)})$ to both sides of (\ref{eq:Keft-integral}) reproduces the ODE (\ref{eq:K1-ode}).
In $G$ variables, this full tadpole contains the local conditional-Wishart contraction $\Omega(K_0)$.
Combining it with the leading conditional-bias reference term gives the identity
\begin{equation}
  \frac12\Hess[K_0]:V_4+B_Q
  =
  \frac12\Hess[K_0]:\left[V_4-\Omega(K_0)\right]
  +B_Q^{\mathrm{res}} .
  \label{eq:tadpole-Wishart-subtracted}
\end{equation}
Replacing $V_4$ by $V_4-\Omega(K_0)$ leaves the tadpole topology unchanged and isolates the connected contribution; $B_Q^{\mathrm{res}}$, however, remains an additional unclosed source not generated by the $G$-only reference action.

By contrast, all Wick contractions arising from the noise vertex $S^{(3)}_{\mathrm{noise}}$ (the $\hat{g}\hat{g}g$ vertex) vanish because either $R(s,s) = 0$ or $\langle\hat{g}\hat{g}\rangle_0 = 0$.
Hence, there is no additional contribution to $\Keft$ from the noise vertex.
The $K_0$--$V_4$--$\Keft$ hierarchy thus emerges as the natural order of the lowest-order loop expansion in the collective bilocal EFT, not as an ad hoc collection of correction equations.

\begin{remark}[Relation to Banta et al.'s $V_4^{(\phi)}$]
\label{rem:bridge}
In the MLP setting considered by Banta et al.~\cite{banta2024structures}, the $V_4$ used in this paper is the kernel fluctuation covariance $V_4^{(G)}$, which is distinct from their preactivation connected four-point function $V_4^{(\phi)}$.
At leading order in large $n$, the law of total covariance gives
\begin{equation}
  V_4^{(G),\ell+1} = \Omega(\bar{K}^{\ell+1}) + V_4^{(\phi),\ell+1} + O(n^{-1}),
  \label{eq:bridge-leading}
\end{equation}
where $\Omega(\bar{K})_{ab,cd} := \bar{K}_{ac}\bar{K}_{bd}+\bar{K}_{ad}\bar{K}_{bc}$ is the conditional-Wishart piece.
The generalized-block decomposition is given in Appendix~\ref{sec:appx-mlp}, and the exact finite-$n$ bridge is derived in Section~\ref{sec:appx-banta}.
\end{remark}

\section[G-only Validation and Validity Diagnostics]{$G$-only Validation and Validity Diagnostics}\label{sec:numerics}

\subsection{Validation of the $G$-only Covariance Theory ($K_0$ and $V_4$)}\label{sec:V4-numerics}

The baseline configuration uses width $n=64$, activation function $\sigma = \tanh$, $\Cw=2$, $\Cb=0$, input dimension $N=4$, and initial kernel $(K_0^0)_{ab} = (1-\rho_{\mathrm{in}})\delta_{ab}\,\kappa + \rho_{\mathrm{in}}\,\kappa$ ($\kappa=2.0$, $\rho_{\mathrm{in}}=0.3$).
The residual scaling is $\eps = 0.05$, depth $L=800$ ($T = \eps^2 L = 2$), and ensemble size $M = 5\times10^6$.
We write $\bar G^\ell:=M^{-1}\sum_I G^{\ell,(I)}$ for the empirical ensemble mean, which estimates the exact mean $\bar K^\ell=\E[G^\ell]$.
The empirical estimator for $V_4^\ell$ is $V_{4,ab,cd}^{\ell,\mathrm{emp}} := n\bigl(\frac{1}{M}\sum_{I} G_{ab}^{\ell,(I)} G_{cd}^{\ell,(I)} - \bar{G}_{ab}^\ell\,\bar{G}_{cd}^\ell\bigr)$.
With the Gaussian initial condition $\phi_i^0 \sim \mathcal{N}(0,K_0^0)$, the Wishart distribution~\cite{muirhead1982aspects} gives $V_{4,ab,cd}^{0,\mathrm{emp}} \approx K_{ac}^0 K_{bd}^0 + K_{ad}^0 K_{bc}^0$.

\begin{figure}[t]
  \centering
  \includegraphics[width=\textwidth]{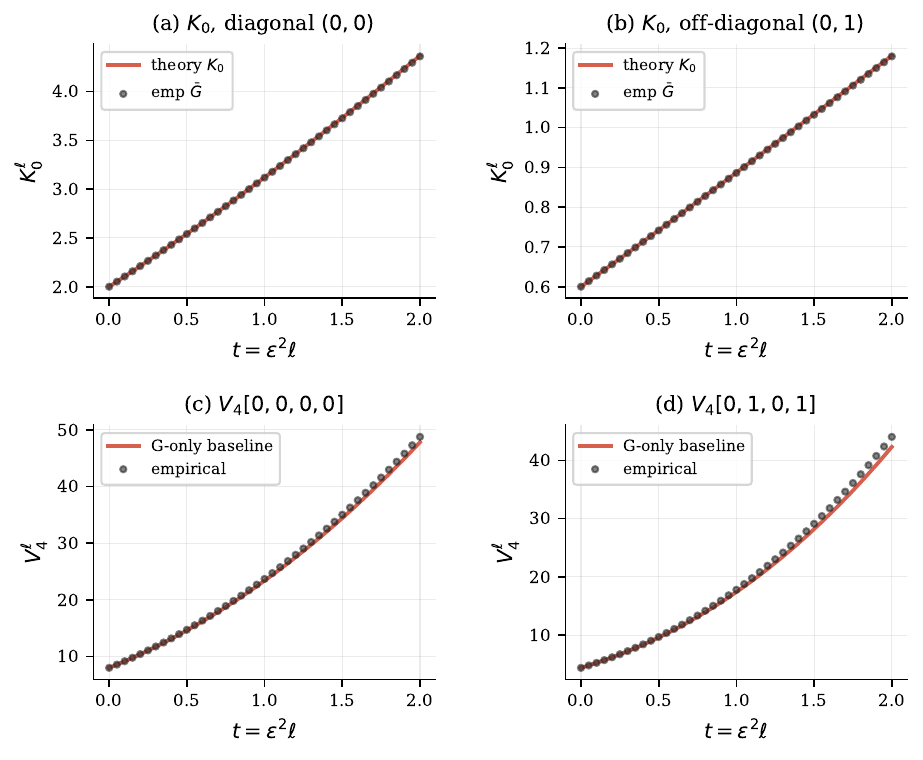}
  \caption{
    \textbf{Trajectories of $K_0$ and $V_4$ in the baseline configuration.}
    Upper: $K_0$ theory (solid) vs.\ empirical $\bar{G}^\ell$ (dots).
    Lower: $V_4$ for the $(0,0,0,0)$ (left) and $(0,1,0,1)$ (right) components.
    The solid line is the $G$-only solution of (\ref{eq:V4-ode}), and the dots are the empirical values $V_{4,\mathrm{emp}}^\ell$.
  }
  \label{fig:K0-V4}
\end{figure}

As shown in Fig.~\ref{fig:K0-V4}, $\bar{G}^\ell$ agrees with $K_0^\ell$ across all tested depths.
For the displayed baseline components, the $G$-only $V_4$ prediction also agrees with the empirical endpoint at the percent level.
A systematic residual nevertheless develops and grows with depth.
To avoid differencing two independently noisy covariance estimates, we pair the pre- and post-update states of each network trajectory and analytically average the Gaussian block innovation conditional on the pre-update state~\cite{blackwell1947conditional,robert2004monte}.
This gives a Rao--Blackwellized next-layer covariance $\widehat V_{4,\mathrm{emp}}^{\ell+1,\mathrm{RB}}$ and the estimator
\begin{equation}
  \widehat R_{V_4}^{\ell,\mathrm{RB}}
  := \frac{\widehat V_{4,\mathrm{emp}}^{\ell+1,\mathrm{RB}}
  -\widehat V_{4,\mathrm{emp}}^\ell}{\varepsilon^2}
  - \left(\chi_{K_0^\ell}\widehat V_{4,\mathrm{emp}}^\ell
  + \widehat V_{4,\mathrm{emp}}^\ell\chi_{K_0^\ell}^\top + \Sigma(K_0^\ell)\right).
  \label{eq:RV4-estimator}
\end{equation}
We evaluate every term in the exact decomposition on the same paired sample; Appendix~\ref{sec:matrix-defect-estimators} gives the estimator identity and fit protocol.
By itself, the residual is not evidence against leading-order Gaussianity: it tests the fluctuation-level projection of the sigma-kernel onto $G$.

\begin{figure}[t]
  \centering
  \includegraphics[width=\textwidth]{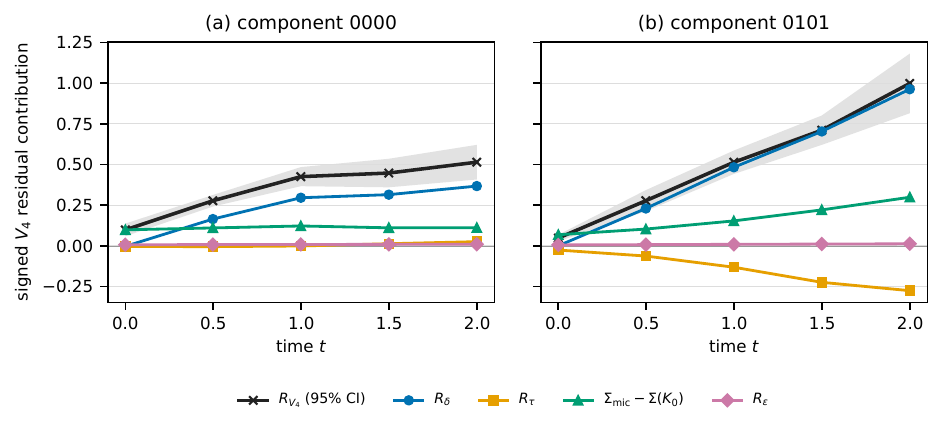}
  \caption{
    \textbf{Direct decomposition of the $G$-only $V_4$ equation residual.}
    The black curve is the Rao--Blackwellized estimator $\widehat R_{V_4}^{\mathrm{RB}}$ in (\ref{eq:RV4-estimator}).
    The remaining curves resolve the exact identity (\ref{eq:V4-defect-decomposition}) into $R_\delta=\mathrm{Cov}(\delta,v)+\mathrm{Cov}(v,\delta)$, $R_\tau=\mathrm{Cov}(\tau,v)+\mathrm{Cov}(v,\tau)$, $\Sigma_{\mathrm{mic}}-\Sigma(K_0)$, and $R_\eps$.
    The first two terms sum to the symmetrized projection defect, $\Delta_{\mathrm{proj}}+\Delta_{\mathrm{proj}}^\top=R_\delta+R_\tau$.
    All four contributions sum to the black total residual $\widehat R_{V_4}^{\mathrm{RB}}$.
    They are measured for components $0000$ and $0101$ in the $N=4$, $n=64$, $\eps=0.05$ baseline.
    Numerical sampling and uncertainty conventions are summarized in Appendix~\ref{sec:numerical-protocols}.
  }
  \label{fig:V4-decomposition}
\end{figure}

To diagnose the source term, we directly measured the exact microscopic quantity $\Sigma_{\mathrm{mic}}^\ell$ defined in (\ref{eq:Sigma-mic-def}), using the empirical sigma-kernel in every term.
The seed-resolved decomposition in Fig.~\ref{fig:V4-decomposition}, including its uncertainty bands, shows that the conditional sigma-kernel contribution $R_\delta$ is the largest growing contribution in both tested components.
The source defect $\Sigma_{\mathrm{mic}}-\Sigma(K_0)$ remains smaller than this conditional sigma-kernel term, although in component $0101$ it grows visibly with depth.
Under the stated large-width expansion, the source defect is expected to enter at the next order in $1/n$.
We therefore treat it as a separately measured correction rather than include it in the leading covariance equation.
At the fixed endpoint $\ell=799$ ($t=1.9975$) and $\eps=0.05$, a linear extrapolation in $1/n$ over $n\in\{32,48,64,96,128,256\}$ gives
\begin{align}
  R_{\delta,0000}^{(\infty)}
  &=0.550\ [0.488,0.617],&
  R_{\delta,0101}^{(\infty)}
  &=0.998\ [0.941,1.060],\notag\\
  R_{V_4,0000}^{(\infty)}
  &=0.570\ [0.501,0.641],&
  R_{V_4,0101}^{(\infty)}
  &=1.048\ [0.984,1.115].
  \label{eq:matrix-width-headline}
\end{align}
Here the superscript $(\infty)$ denotes the fitted $n\to\infty$ intercept, and brackets denote bootstrap 95\% intervals.
The corresponding $R_\tau$ intercepts are
\begin{equation}
  R_{\tau,0000}^{(\infty)}=-0.0023\ [-0.0182,0.0134],
  \qquad
  R_{\tau,0101}^{(\infty)}=0.0171\ [-0.0002,0.0358],
\end{equation}
and are compatible with zero, as are the microscopic-source intercepts.
Consequently, the symmetrized projection defect $\Delta_{\mathrm{proj}}+\Delta_{\mathrm{proj}}^\top=R_\delta+R_\tau$ is consistent with the same nonzero limiting contribution as $R_\delta$ in these two components.
Independent tests at $\eps=0.10,0.07,0.05$ show the expected $O(\eps^2)$ scaling of $R_\eps$.
The width extrapolation and step-size check are shown in Appendix~\ref{sec:matrix-defect-estimators}.
In both analyzed components, the width-extrapolated residual is therefore traced to conditional sigma-kernel information absent from the retained equal-time Markov state.
Appendix~\ref{sec:appx-validation} reports absolute predictive errors under parameter and activation changes, including a smooth ReLU-like Softplus diagnostic.

\subsection{Validation of $\Keft$ and the Conditional-Bias Source}\label{sec:source-diag}

We estimate $\Kmic^\ell$ by $n(\bar{G}^\ell-K_0^\ell)$ and compare it with the EFT prediction $\Keft^\ell$ and the source-integrated reconstruction
\begin{equation}
  K_{1,\mathrm{src}}^{\ell+1}
  =K_{1,\mathrm{src}}^\ell+\eps^2\Cw\,\Uex^\ell,
  \qquad K_{1,\mathrm{src}}^0=0,
  \label{eq:K1-source-reconstruction}
\end{equation}
which inserts the measured exact source into (\ref{eq:K1mic-recursion}).
Using the $G$-only solution $V_4^\ell$ of (\ref{eq:V4-ode}), the EFT model source is defined as
\begin{equation}
  \Umod^\ell :=
    \frac{1}{\Cw}\chi_{K_0^\ell}[\Keft^\ell]
    + \frac{1}{2\Cw}\Hess[K_0^\ell]:V_4^\ell
  \label{eq:Umod-def}
\end{equation}
(satisfying $\Keft^{\ell+1} = \Keft^\ell + \eps^2\Cw\,\Umod^\ell$).
Because $K_{1,\mathrm{src}}$ is driven by the measured $\Uex$, its agreement with $\Kmic$ isolates the discrepancy to the modeled source $\Umod$.
Figure~\ref{fig:K1-main} compares the empirical $K_1$ trajectory with the $G$-only reference, the Wishart-subtracted candidate, and $K_{1,\mathrm{src}}$ from (\ref{eq:K1-source-reconstruction}).

\begin{figure}[t]
  \centering
  \includegraphics[width=\textwidth]{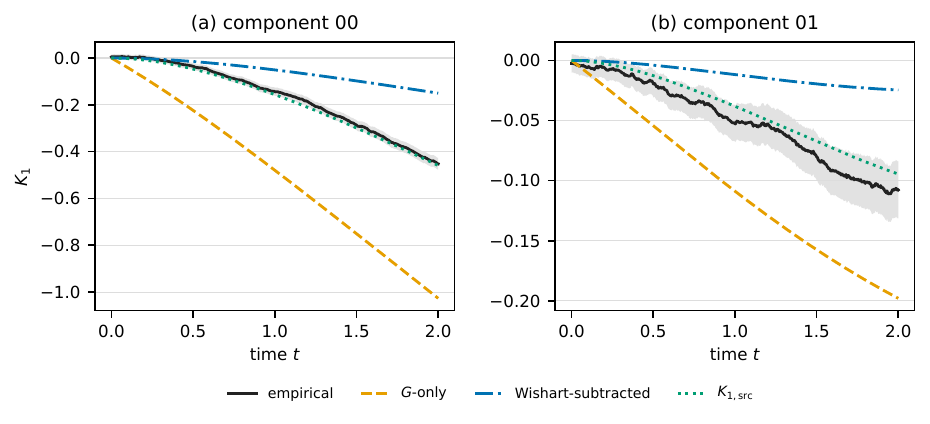}
  \caption{
    \textbf{NLO mean trajectories and measured-source reconstruction.}
    The empirical estimate $n(\bar{G}^\ell-K_0^\ell)$ is compared with the original $G$-only prediction, the Wishart-subtracted candidate from (\ref{eq:K1-Wishart-subtracted}) with $B_Q^{\mathrm{res}}=0$, and $K_{1,\mathrm{src}}^\ell=\sum_{m<\ell}\eps^2\Cw\,\Uex^m$ with the measured exact source $\Uex^m=n[\bar S^m-E_2(K_0^m)]$.
    The measured-source reconstruction is a closure check, not an independent predictive theory.
    Left: diagonal component $(0,0)$; right: off-diagonal $(0,1)$.
  }
  \label{fig:K1-main}
\end{figure}

At the endpoint, the empirical value and the original $G$-only and Wishart-subtracted predictions are, respectively,
\begin{equation}
  (-0.4522,-1.0267,-0.1501)\quad(00),
  \qquad
  (-0.1078,-0.1980,-0.0246)\quad(01).
\end{equation}
Wishart subtraction reduces the depth-wise RMSE by $60\%$ and $22\%$, respectively, but the corrected trajectory crosses to the opposite side of the empirical one, leaving endpoint absolute differences of $0.3021$ and $0.0832$.
Thus it accounts for a substantial part, but not all, of the original $G$-only error.

\begin{remark}\label{rem:U1-zero}
  When $\phi_i^0$ are i.i.d.\ $\mathcal{N}(0, K_0^0)$, $\Uex^0 = 0$ holds exactly at finite $n$.
  Indeed, $\E[S_{ab}^0] = E_2(K_0^0)_{ab}$, so
  $\Uex^0 = n(\bar{S}^0 - E_2(K_0^0)) = 0$.
\end{remark}

However, the EFT model source at $\ell=0$ (where $\Keft^0=0$) is nonzero:
\begin{equation}
  \Umod^0 = \frac{1}{2\Cw}\Hess[K_0^0]:V_4^0 \neq 0,
  \label{eq:Umod-init}
\end{equation}
as confirmed numerically in our setup.
The missing term is the conditional-bias defect (\ref{eq:BQ-def}).
For i.i.d.\ Gaussian initialization, let $F$ denote the conditional sigma-kernel mean function satisfying $F(G^0)=\E[S^0\mid G^0]$.
The random variable $F(G^0)$ is generally not equal sample by sample to the Gaussian function $E_2(G^0)$.
Nevertheless, the tower property and the exact Gaussian marginal give
\begin{equation}
  \E[F(G^0)]=\E[S^0]=E_2(K_0^0).
  \label{eq:conditional-source-init}
\end{equation}
Since $\E[G^0]=K_0^0$ and $G^0-K_0^0=O_p(n^{-1/2})$, the second-order delta-method expansion~\cite{vandervaart1998asymptotic} gives
\begin{equation}
  \E[E_2(G^0)]
  =E_2(K_0^0)
   +\frac{1}{2n}D^2E_2(K_0^0):V_4^0+o(n^{-1}).
  \label{eq:E2-Wishart-expansion}
\end{equation}
Substituting (\ref{eq:conditional-source-init})--(\ref{eq:E2-Wishart-expansion}) into $B_{Q,n}^0=n\Cw\,\E[F(G^0)-E_2(G^0)]$, and using $D^2Q=\Cw D^2E_2$, yields
\begin{equation}
  \frac12\Hess[K_0^0]:V_4^0+B_{Q,n}^0=o(1).
  \label{eq:BQ-init-cancel}
\end{equation}
At finite $n$, $\Uex^0=0$ exactly by Remark~\ref{rem:U1-zero}; at retained NLO order, (\ref{eq:BQ-init-cancel}) shows how the limiting defect $B_Q^0$ cancels the spurious $G$-only tadpole.
At later depths a measured $B_{Q,n}^\ell$ localizes the source defect but is not itself a predictive closure.
At $n=64$, the endpoint remainders are
\begin{equation}
  B_{Q,n;00}^{\mathrm{res}}=-0.203\ [-0.237,-0.168],
  \qquad
  B_{Q,n;01}^{\mathrm{res}}=-0.045\ [-0.075,-0.015],
  \label{eq:BQ-res-endpoint}
\end{equation}
with two-sided Student-$t$ 95\% intervals across eight independent seeds.
The diagonal remainder starts from a value compatible with zero and tends toward a plateau at large depth; the off-diagonal remainder approaches a smaller, nearly flat nonzero value after its initial transient.
Appendix~\ref{sec:conditional-wishart-bias-appx} gives the complete depth-wise decomposition in Figure~\ref{fig:BQ-decomposition-appx}.
At this tested finite width, the equal-time Markov closure based on the retained $K_0$, $K_1$, and $V_4$ variables therefore remains incomplete.

\section{Constructive Scalar Sigma-Kernel Closure}\label{sec:hermite}

Sections~\ref{sec:V4-numerics} and~\ref{sec:source-diag} localized the two validity limits of the $G$-only description to two effects of the unresolved sigma-kernel sector:
the covariance projection defect $\Delta_{\mathrm{proj}}$ and the conditional-bias defect $B_Q$.
The former is dominated by the conditional sigma-kernel component in the tested matrix entries, while the latter controls the NLO mean source.
We now construct a predictive closure for the scalar diagonal ($N=1$) $\tanh$ sector that retains this missing information while preserving the finite residual step.
We perform the large-width expansion at fixed residual step and fixed Hermite truncation, retain the leading fluctuation and NLO mean orders, and then test convergence as the truncation and quadrature orders increase.
The aim is to compute the scalar projection correction $\Delta_{\mathrm{proj}}$ defined in (\ref{eq:projection-defect}) and the asymptotic conditional-bias term $B_Q$ defined from (\ref{eq:BQ-def}), without inserting measured finite-width defects.

\subsection{Hermite--Galerkin Construction}

\paragraph{Scalar transition and Hermite coordinates.}
Suppressing the layer index, the exact scalar block law is
\begin{equation}
  \phi_i^+=\phi_i+\sqrt{\Delta t\,\hatQ}\,z_i,\qquad
  \hatQ=\Cb+\Cw S,\qquad
  S=\frac1n\sum_{j=1}^n\tanh^2(\phi_j),
  \label{eq:hermite-scalar-block}
\end{equation}
where $\Delta t=\eps^2$ and $z_i\stackrel{\mathrm{i.i.d.}}{\sim}\mathcal N(0,1)$.
Let $q=K_0$ be the leading variance and define
\begin{align}
  c_0(q)&:=\E[\tanh^2(\sqrt q\,X)],\qquad X\sim\mathcal N(0,1),\notag\\
  Q(q)&:=\Cb+\Cw c_0(q),\qquad
  q^+:=q+\Delta t\,Q(q).
  \label{eq:hermite-background}
\end{align}
Thus $q$ is the deterministic leading background, not the finite-width empirical variance.
We standardize by $x_i=\phi_i/\sqrt q$ and define $f_q(x):=\tanh^2(\sqrt q\,x)$, so that the empirical sigma-kernel is exactly
\begin{equation}
  S=\frac1n\sum_{i=1}^n f_q(x_i).
  \label{eq:hermite-sigma-empirical}
\end{equation}
We use the standard normalized Gaussian Hermite basis~\cite{janson1997gaussian}, $h_k(x):=\operatorname{He}_k(x)/\sqrt{k!}$, for which
\begin{equation}
  \E[h_k(X)h_m(X)]=\delta_{km},\qquad
  h_k''(x)=\sqrt{k(k-1)}\,h_{k-2}(x).
  \label{eq:hermite-orthogonality}
\end{equation}
Because $f_q$ is even, only even Hermite modes occur:
\begin{equation}
  f_q(x)=c_0(q)+\sum_{r\ge1}c_r(q)h_{2r}(x),\qquad
  c_r(q):=\E[f_q(X)h_{2r}(X)].
  \label{eq:hermite-coefficients}
\end{equation}
Applying a finite Hermite--Galerkin truncation~\cite{xiu2002wiener} at order $R$ (maximum retained Hermite degree $p_{\max}=2R$), define
\begin{equation}
  \begin{aligned}
    a_r&:=\frac1n\sum_{i=1}^n h_{2r}(x_i),&
    \boldsymbol a_R&:=(a_1,\ldots,a_R)^\top,\\
    S&\simeq c_0(q)+\boldsymbol c_R(q)^\top\boldsymbol a_R,&
    \boldsymbol c_R(q)&:=(c_1(q),\ldots,c_R(q))^\top .
  \end{aligned}
  \label{eq:hermite-modes}
\end{equation}
The first mode directly represents the empirical kernel because $h_2(x)=(x^2-1)/\sqrt2$:
\begin{equation}
  G=\frac1n\sum_i\phi_i^2=q(1+\sqrt2\,a_1).
  \label{eq:G-hermite-mode}
\end{equation}
Assuming the large-width expansion
\begin{equation}
  \E[a_r]=\frac{m_r}{n}+o(n^{-1}),
\end{equation}
we define $m_r$ as the asymptotic NLO mean coefficient and resolve each empirical mode as
\begin{equation}
  a_r=\frac{m_r}{n}+\frac{\xi_r}{\sqrt n}+o_p(n^{-1/2}),
  \qquad \E[\xi_r]=0,\qquad C_{rs}:=\E[\xi_r\xi_s].
  \label{eq:hermite-scaling}
\end{equation}
Write $\boldsymbol m_R:=(m_1,\ldots,m_R)^\top$ and $\boldsymbol\xi_R:=(\xi_1,\ldots,\xi_R)^\top$.
At Gaussian initialization, orthonormality gives $\boldsymbol m_R^0=0$ and $C^0=I_R$, so the Hermite hierarchy needs no measured finite-width input.

\paragraph{One-step mode map.}
Within the finite-$R$ Hermite projection, dividing (\ref{eq:hermite-scalar-block}) by $\sqrt{q^+}$ gives the approximate standardized update
\begin{equation}
  x_i^+\simeq\rho x_i+\sqrt{1-\rho^2+\delta\lambda_R}\,z_i.
  \label{eq:hermite-standardized-step}
\end{equation}
\begin{equation}
  \rho:=\sqrt{\frac q{q^+}},\qquad
  \delta\lambda_R:=\gamma\,\boldsymbol c_R(q)^\top\boldsymbol a_R,\qquad
  \gamma:=\frac{\Delta t\,\Cw}{q^+}.
  \label{eq:hermite-standardized-parameters}
\end{equation}
Here $1-\rho^2=\Delta t\,Q(q)/q^+$ is the Gaussian-reference innovation variance, while $\delta\lambda_R=O_p(n^{-1/2})$ is the common variance perturbation generated by the retained sigma-kernel modes.
We suppress the subscript $R$ on retained vectors and on $\delta\lambda_R$ below when no confusion can arise.
For $\delta\lambda_R=0$, Mehler's identity~\cite{janson1997gaussian} diagonalizes the transition:
\begin{equation}
  \E_z[h_k(\rho x+\sqrt{1-\rho^2}\,z)\mid x]=\rho^k h_k(x).
  \label{eq:mehler}
\end{equation}
To include $\delta\lambda$, set $\mathcal T_k(x;\lambda):=\E_z[h_k(\rho x+\sqrt\lambda z)]$.
Gaussian integration by parts and (\ref{eq:hermite-orthogonality}) give
\begin{equation}
  \partial_\lambda\mathcal T_k
  =\frac12\sqrt{k(k-1)}\,\mathcal T_{k-2},\qquad
  \partial_\lambda^2\mathcal T_k
  =\frac14\sqrt{k(k-1)(k-2)(k-3)}\,\mathcal T_{k-4}.
  \label{eq:hermite-variance-derivatives}
\end{equation}
Taylor expansion at $\lambda=1-\rho^2$, followed by empirical averaging, therefore yields
\begin{align}
  a_r^+={}&\rho^{2r}a_r
  +\alpha_r\gamma(\boldsymbol c^\top\boldsymbol a)a_{r-1}
  +\beta_r\gamma^2(\boldsymbol c^\top\boldsymbol a)^2a_{r-2}
  +\frac{\omega_r}{\sqrt n}+\mathcal R_r,\notag\\
  \alpha_r:={}&\frac12\sqrt{(2r)(2r-1)}\,\rho^{2r-2},\qquad
  \beta_r:=\frac18\sqrt{(2r)(2r-1)(2r-2)(2r-3)}\,\rho^{2r-4},
  \label{eq:hermite-mode-map}
\end{align}
with $a_0:=1$.
The martingale innovation is explicitly
\begin{equation}
  \omega_r:=\frac1{\sqrt n}\sum_{i=1}^n
  \left[
    h_{2r}(x_i^+)-\E_z[h_{2r}(x_i^+)\mid\cF]
  \right],
  \qquad \E[\omega_r\mid\cF]=0.
  \label{eq:hermite-martingale}
\end{equation}
Write $\boldsymbol\omega:=(\omega_1,\ldots,\omega_R)^\top$.
The remainder $\mathcal R_r$ contains terms beyond the retained leading-fluctuation/NLO-mean order as well as the effect of the discarded Hermite tail in $S$.
The former vanish in the large-$n$ expansion at fixed $R$, whereas the Galerkin tail is controlled only by increasing $R$.

\paragraph{Leading fluctuation covariance.}
In (\ref{eq:hermite-mode-map}), the first variance-perturbation term contributes at $O(n^{-1/2})$ only for $r=1$, because then $a_{r-1}=a_0=1$; for $r\ge2$ it is $O(n^{-1})$ and becomes an NLO mean source.
Consequently
\begin{align}
  \boldsymbol\xi^+&=\mathcal A\boldsymbol\xi+\boldsymbol\omega+o_p(1),\notag\\
  \mathcal A_{rs}&=\rho^{2r}\delta_{rs}
    +\delta_{r1}\frac{\Delta t\,\Cw}{\sqrt2\,q^+}c_s(q).
  \label{eq:hermite-transport}
\end{align}
For the Gaussian-reference transition ($\delta\lambda=0$), orthogonality and (\ref{eq:mehler}) imply, at retained leading-fluctuation order,
\begin{equation}
  D_{rs}:=\E[\omega_r\omega_s]
  =\delta_{rs}(1-\rho^{4r}),\qquad
  \E[\boldsymbol\omega\boldsymbol\xi^\top]=0.
  \label{eq:hermite-innovation}
\end{equation}
Taking the covariance of (\ref{eq:hermite-transport}) closes the leading-fluctuation sector of the Hermite hierarchy:
\begin{equation}
  C^+=\mathcal A C\mathcal A^\top+D.
  \label{eq:hermite-cov}
\end{equation}

\paragraph{NLO mean.}
Multiplying the expectation of (\ref{eq:hermite-mode-map}) by $n$ uses $n\E[a_sa_{r-1}]=C_{s,r-1}+o(1)$.
The first-order variance perturbation therefore produces
\begin{equation}
  b_1(C)=0,\qquad
  b_r(C)=
  \frac{\Delta t\,\Cw}{2q^+}\sqrt{(2r)(2r-1)}\,\rho^{2r-2}
  \sum_{s=1}^R c_s(q)C_{s,r-1},\quad r\ge2.
  \label{eq:hermite-b-source}
\end{equation}
For $r=2$, the second-order term in (\ref{eq:hermite-mode-map}) also survives because $a_{r-2}=a_0=1$:
\begin{equation}
  b_2^{\mathrm{full}}(C)
  :=b_2(C)
    +\frac{\sqrt{24}}8\frac{\Delta t^2\Cw^2}{(q^+)^2}
     \boldsymbol c(q)^\top C\boldsymbol c(q).
  \label{eq:hermite-finite-step-source}
\end{equation}
It is $O(\Delta t^2)=O(\eps^4)$ and vanishes in the continuous-depth limit, but must be retained for the finite residual step used here.
In the following, the second component of $\boldsymbol b$ denotes
$b_2^{\mathrm{full}}$.
The NLO mean update is thus
\begin{equation}
  \boldsymbol m^+=\mathcal A\boldsymbol m+\boldsymbol b(C).
  \label{eq:hermite-mean}
\end{equation}

\paragraph{Reconstruction and consistency.}
The hierarchy targets the scalar specializations of $\Kmic$, $\Uex$, and $V_4$.
Equations (\ref{eq:G-hermite-mode}) and (\ref{eq:hermite-modes}) give the finite-truncation predictions
\begin{equation}
  V_{4,R}^{\mathrm H}=2q^2C_{11},\qquad
  K_{1,R}^{\mathrm H}=\sqrt2\,q\,m_1,\qquad
  U_{1,R}^{\mathrm H}=\boldsymbol c_R(q)^\top\boldsymbol m_R.
  \label{eq:hermite-reconstruction}
\end{equation}
When no confusion can arise, we suppress the labels $R$ and $\mathrm H$ below.
Since $\rho^2=q/q^+$, the $r=1$ component of (\ref{eq:hermite-mean}) gives
\begin{equation}
  K_1^+=\sqrt2\,q^+m_1^+
  =\sqrt2\,q\,m_1+\Delta t\,\Cw\boldsymbol c^\top\boldsymbol m
  =K_1+\Delta t\,\Cw U_1.
  \label{eq:hermite-mean-consistency}
\end{equation}
Equation~(\ref{eq:hermite-mean-consistency}) is the finite-$R$ scalar counterpart of the exact finite-$n$ recursion (\ref{eq:K1mic-recursion});
thus the Hermite hierarchy preserves its structure at retained NLO order.
It also reconstructs the two closure defects.
In particular, the Hermite hierarchy relaxes (PROJ) by evolving the sigma-kernel modes required to predict $C_{\nu v}$, and hence $\Delta_{\mathrm{proj}}$, rather than setting this defect to zero.
Writing $\nu=\sqrt n(S-\E[S])\simeq\boldsymbol c^\top\boldsymbol\xi$ and $v=\sqrt2q\,\xi_1$, and defining
\begin{equation}
  \chi_q:=Q'(q)=\Cw c_0'(q),\qquad Q''(q)=\Cw c_0''(q),
  \label{eq:hermite-scalar-derivatives}
\end{equation}
the scalar specialization of the projection defect defined in (\ref{eq:projection-defect}) is
\begin{equation}
  \Delta_{\mathrm{proj},R}^{\mathrm H}
  =\Cw\sqrt2q\sum_{r=1}^R c_r(q)C_{r1}
   -\chi_qV_{4,R}^{\mathrm H}.
  \label{eq:hermite-projection-defect}
\end{equation}
Moreover, within the finite-$R$ hierarchy,
\begin{align}
  \E_{\mathrm H}[\hatQ]
  &=Q(q)+\frac{\Cw}{n}U_{1,R}^{\mathrm H}+o(n^{-1}),\notag\\
  \E_{\mathrm H}[Q(G)]
  &=Q(q)+\frac1n\left[
    \chi_qK_{1,R}^{\mathrm H}
    +\frac12Q''(q)V_{4,R}^{\mathrm H}
  \right]+o(n^{-1}).
  \label{eq:hermite-Q-expectations}
\end{align}
Here $\E_{\mathrm H}$ denotes expectation within the finite-$R$ hierarchy.
Hence the conditional-bias defect defined in (\ref{eq:BQ-def}) becomes
\begin{equation}
  B_{Q,R}^{\mathrm H}
  =\Cw U_{1,R}^{\mathrm H}-\chi_qK_{1,R}^{\mathrm H}
   -\frac12Q''(q)V_{4,R}^{\mathrm H}.
  \label{eq:BQ-hermite}
\end{equation}
Here $Q''(q)V_4$ is the scalar specialization of $\Hess[K_0]:V_4$.
The quantities in (\ref{eq:hermite-projection-defect}) and (\ref{eq:BQ-hermite}) are reconstructed diagnostics of the solved $(C,\boldsymbol m)$ Hermite hierarchy, not additional independently evolved state variables.
Equations (\ref{eq:hermite-background})--(\ref{eq:BQ-hermite}) form a finite-dimensional predictive Hermite--Galerkin closure for the scalar diagonal $\tanh$ sector at leading fluctuation and NLO mean order.
They are not an exact representation of the complete finite-width dynamics; the remaining approximations are the finite Hermite projection, the retained $1/n$ orders, and the scalar diagonal $\tanh$ specialization.

\paragraph{Relation to the collective EFT.}
The retained Hermite modes provide additional collective coordinates for the sigma sector.
The kernel fluctuation is the first coordinate $v=\sqrt2q\,\xi_1$, while the higher modes retain information eliminated by the $G$-only projection.
The linear recursion (\ref{eq:hermite-transport})--(\ref{eq:hermite-cov}) defines the response and covariance propagation in this enlarged finite-dimensional state space.
The quadratic mode coupling in (\ref{eq:hermite-mode-map}) feeds the covariance $C$ into the NLO mean source $\boldsymbol b(C)$, in direct analogy with the tadpole source of the $G$-only reference EFT.

We regard this correspondence only as a field-theoretic interpretation.
A full matrix-valued joint $(G,S)$ action, and the generally nonlocal reduced dynamics obtained by eliminating the sigma sector, are outside the scope of the present work.

\subsection{Numerical Validation of the Scalar Closure}
The Hermite predictions are stable under quadrature- and truncation-order refinement, and the propagated covariance remains positive semidefinite throughout the tested trajectory (Table~\ref{tab:hermite-convergence}).
At the empirical endpoint, the predicted $K_1$ and $U_1$ values lie in both the linear and quadratic variance-reduced width-extrapolation intervals (Table~\ref{tab:width-fits}).
For $V_4$, the Hermite value lies inside the linear interval and close to the lower boundary of the quadratic interval; relative to the corresponding fit centers, it substantially reduces the $G$-only endpoint error.
The fixed-width paired-reference trajectories in Figs.~\ref{fig:hermite-k1} and~\ref{fig:hermite-v4} provide a complementary comparison:
across the nonzero sampled layers, the absolute theory--Monte Carlo discrepancy is at most $0.41$ and $1.01$ times the standard error across five independent seeds for $K_1$ and $U_1$, respectively.
For $V_4$, the Hermite hierarchy markedly reduces the discrepancy relative to the $G$-only baseline, although a residual finite-width offset remains.
The Hermite value is $0.0069$ below the lower edge of the quadratic-fit interval, much smaller than the $0.0708$ shift between the linear and quadratic fit centers.
Because the hierarchy is stable under quadrature and truncation refinement, this small difference is consistent with the remaining width-extrapolation model-form sensitivity.
Appendix~\ref{sec:appx-validation} reports the endpoint fits and broader robustness tests.

\begin{figure}[t]
  \centering
  \includegraphics[width=0.95\textwidth]{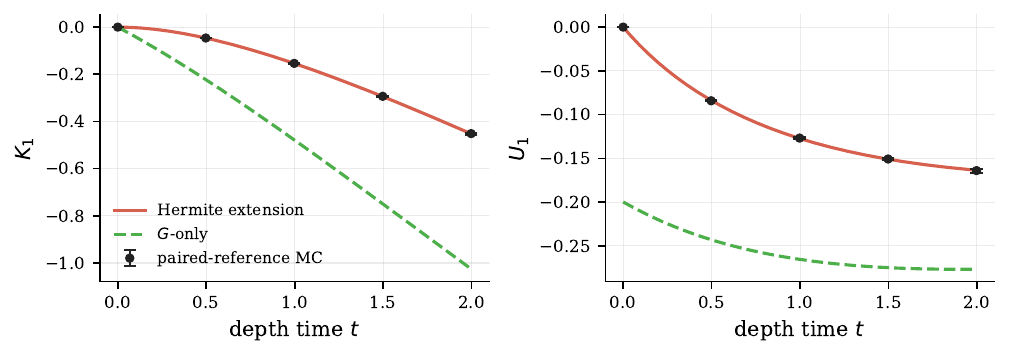}
  \caption{
    \textbf{Scalar Hermite prediction of $K_1$ and $U_1$.}
    Converged scalar Hermite and $G$-only reference predictions against paired Gaussian-reference Monte Carlo estimates in the $N=1$ sector at width 64.
    Sampling details are reported in Appendix~\ref{sec:numerical-protocols}.
  }
  \label{fig:hermite-k1}
\end{figure}

\begin{figure}[t]
  \centering
  \includegraphics[width=0.95\textwidth]{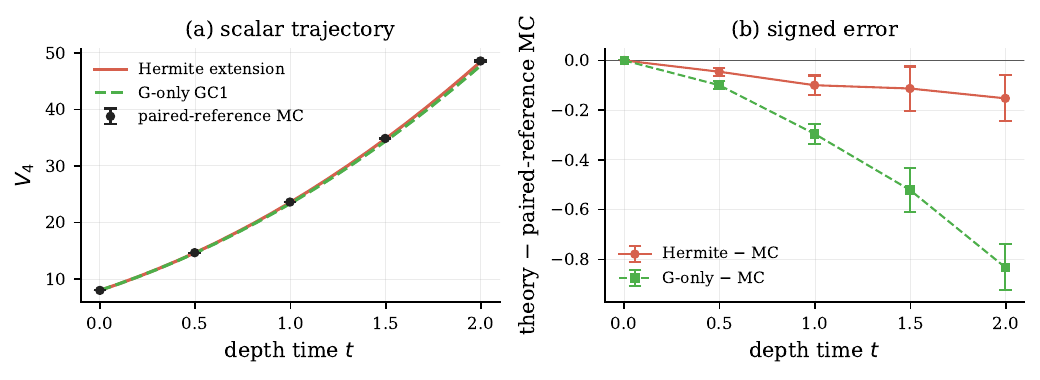}
  \caption{
    \textbf{Scalar Hermite prediction of $V_4$.}
    The scalar Hermite $V_4$ trajectory, the equivalent diagonal component of the $N=4$ matrix $G$-only baseline, and the paired-reference estimator, together with signed finite-width discrepancies.
    The corresponding width-extrapolated endpoint comparison is given in Table~\ref{tab:width-fits}.
  }
  \label{fig:hermite-v4}
\end{figure}

\section{Discussion}\label{sec:disc}

The tested accuracy and the remaining corrections in $K_0$, $V_4$, and $K_1$ are associated with distinct approximation stages (Table~\ref{tab:failure}).

\begin{table}[htbp]
  \centering
  \caption{Validity and additional information required at each order}
  \label{tab:failure}
  \begin{tabular}{lp{5.5cm}l}
    \hline
    Quantity & Assessment in the tested regime & Onset of correction \\
    \hline
    $K_0$ & Accurate; (GC0) is sufficient at leading order & --- \\
    $V_4$ & Percent-level $G$-only baseline; in two tested components the large-width residual is localized to $R_\delta$ & Grows with depth \\
    $K_1$ & Wishart subtraction improves the baseline, but $B_{Q,n}^{\mathrm{res}}$ remains at later depths in the finite-width data & From finite depth \\
    \hline
  \end{tabular}
\end{table}

\paragraph{Defect interpretation.}
Equation (\ref{eq:V4-defect-decomposition}) gives a measurable interpretation of the residual:
the leading mean closure can remain accurate while the cross-covariance between sigma-kernel and kernel fluctuations is not represented by the $G$-only projection.
For $K_1$, the exact source recursion remains valid, but the baseline tadpole is incomplete because it omits $B_Q$.
In the scalar setting, the Hermite hierarchy improves $V_4$, $K_1$, and $U_1$, and yields $B_Q$ as a derived diagnostic of the solved hierarchy.
The matrix projection defect is measured independently in Section~\ref{sec:V4-numerics}.

\paragraph{MLP control and residual-stream memory.}
To distinguish a generic finite-width Gaussian-layer effect from residual carry-over, we repeated both diagnostics in the matched $\alpha=0$ MLP limit using the exact discrete MLP recurrence.
The large-width intercepts of the non-Wishart conditional-bias remainder and the total $V_4$ recurrence residual are compatible with zero in both tested components (Appendix~\ref{sec:mlp-control}).
For the generalized block, the standardized same-neuron carry-over coefficient is $\rho=\alpha\sqrt{q/q^+}$, and the Gaussian reference transition transports a Hermite mode as $h_k\mapsto\rho^k h_k$.
Thus $\alpha=0$ resets every nonconstant single-neuron mode at the next layer, whereas the residual stream retains higher sigma-kernel modes across blocks.
Within the matched simplified model and parameter regime tested here, the contrast is consistent with, and strongly supports, a same-neuron multi-time-memory origin of the persistent ResNet defects.

\paragraph{Machine-learning significance.}
The relations in (\ref{eq:finite-width-representation-statistics}) separate two distinct leading finite-width effects omitted by the deterministic infinite-width kernel.
For an individual kernel entry, the systematic shift and standard-deviation scales are $\frac{1}{n}K_{1,ab}^\ell$ and $\sqrt{V_{4,ab,ab}^\ell/n}$, respectively.
For example, the width-normalized squared distance $\mathcal D_{ab}^\ell=G_{aa}^\ell+G_{bb}^\ell-2G_{ab}^\ell$ between two hidden representations is a linear functional of $G^\ell$.
The corresponding contractions of $K_1^\ell$ and $V_4^\ell$ therefore determine, respectively, its leading ensemble-mean shift and its variance across random initializations.

An accurate trajectory for $K_0$ therefore does not by itself imply that a finite-width network is close to its infinite-width description in distribution.
Our results further show that the two corrections require different information: the covariance correction depends on the projection of sigma-kernel fluctuations onto kernel fluctuations, whereas the mean correction contains the conditional-bias term $B_Q$.
The failure of the $G$-only closure is thus not only a numerical discrepancy; it identifies sigma-kernel information not represented by the retained equal-time Markov state.
In a residual architecture, the omitted per-block contributions are transported and may accumulate with depth.
The scalar Hermite hierarchy provides a constructive example of an enlarged state space that recovers the corresponding scalar finite-width corrections.

We restrict this interpretation to initialization-time representation geometry.
Training dynamics, output propagation, predictive uncertainty, generalization, and trained performance are outside the scope of the present work.

\paragraph{Scope and relation to prior work.}
Our contribution comprises the increment-based exact ResNet block law, the separation of exact identities from the local $G$-only reference EFT, the localization of its projection and conditional-bias defects, and a predictive scalar sigma-kernel closure.
The architecture is the simplified fully connected pre-activation residual block in (\ref{eq:resnet-block}).
The exact block and kernel recursions extend to nonsmooth activations with finite required moments, including ReLU, whereas the derivative-based NLO formulas require smoothness.
The full $K_1/U_1/V_4$ Hermite validation applies to scalar diagonal $\tanh$;
Appendix~\ref{sec:appx-validation} tests the $G$-only baseline for $\operatorname{erf}(x/\sqrt2)$ and uses Softplus as a smooth family approaching ReLU to test a scalar all-mode covariance correction.
These results do not provide a quantitative ReLU NLO theory.
The present $1/n$ expansion is formulated at fixed residual time $T$.
Although the observed growth of $V_4(T)$ increases the finite-width covariance scale $V_4(T)/n$, it does not determine a coupled scaling $T=T(n)$ or control the higher-order remainder uniformly in $T$.
A quantitative connection to the joint depth--width limits studied by Li, Nica, and Roy~\cite{li2021future_log_gaussian,li2022neural_covariance_sde} would require a separate coupled-limit analysis.
More broadly, related field-theoretic questions are under active investigation, including decompositions of network dynamics into architecture- and data-dependent parts~\cite{lee2025bulk_boundary} and structural obstructions to realizing field theories as finite network ensembles~\cite{harvey2026what,ferko2026neural}.

\paragraph{Future work.}
A matrix-valued $(G,S)$ hierarchy, a nonsmooth formulation adapted to ReLU, and extensions to convolutional, normalized, or otherwise standard ResNet blocks remain open.
Integrating out the sigma sector may also produce memory and colored-noise kernels absent from the local reference action.

\section{Conclusion}\label{sec:conc}

We develop a theory of finite-width initialization dynamics for a simplified pre-activation residual model.
The construction starts from the exact conditional Gaussian law of the increment and distinguishes exact identities from closure and continuous-depth approximations.

The normalized conditional Gaussian transition admits the exact discrete response-field action (\ref{eq:MSRJD}) without auxiliary Grassmann fields.
Using the increment as the primary variable separates skip carry-over from residual noise and exposes the continuous-depth scaling.
Within the local $G$-only reference EFT, the retained finite-width contributions are organized into the noise source $\Sigma$, the transport $\chi$, and the drift curvature $\Hess$, whose contraction with $V_4$ generates the tadpole source.
The identity $\Uex^\ell=n(\bar{S}^\ell-E_2(K_0^\ell))$ holds at finite $n$ around the chosen background $K_0$.

The continuous-depth $K_0$ equation follows from (GC0).
The $V_4$ baseline additionally uses (PROJ) and leading source-background evaluation, while the reference $\Keft$ equation further uses the second-order delta expansion and (GC1).
Within the local reference action, $\Keft$ is the one-loop tadpole of the drift cubic vertex, while the noise cubic vertex vanishes under the It\^o convention.

Numerically, $K_0$ remains accurate across the tested settings.
In the tanh/erf tests, the $G$-only $V_4$ endpoint predictions agree with Monte Carlo at the percent-to-few-percent level.
A systematic conditional sigma-kernel projection contribution nevertheless develops with depth and has a nonzero width-extrapolated intercept in the two tested matrix components, while the Taylor/background and source contributions extrapolate to zero.
The leading conditional-Wishart term removes the spurious initial $K_1$ tadpole and improves the trajectory, but the finite-width measurements retain a non-Wishart conditional-bias remainder at later depths.
A matched exact-discrete MLP control is consistent with vanishing large-width residuals in the tested components and supports a residual-stream-memory interpretation.
A finite scalar diagonal $\tanh$ Hermite--Galerkin closure incorporates the enlarged sigma-kernel sector, agrees with variance-reduced $K_1/U_1$ estimates, and substantially reduces the endpoint $V_4$ error.
A complementary Softplus study identifies the contribution of odd Hermite modes and incorporates them in a scalar all-mode covariance extension.
The scalar extensions apply at retained fluctuation/NLO mean order; matrix-valued and nonsmooth extensions remain future work.
At the level of initialization-time representation geometry, $K_1$ and $V_4$ quantify the leading systematic shift and covariance across random initializations, while the sigma-kernel analysis identifies the additional collective information needed to predict these finite-width effects.

\appendix
\section{Unified Description of ResNets and MLPs, and Correspondence with Banta et al.}\label{sec:appx-mlp}

The main text develops the $G$-only formulation for the ResNet block $(\alpha,\eps)=(1,\eps\ll1)$.
Here we extend it to the generalized block $\phi_i^{\ell+1}=\alpha\phi_i^\ell+\eps\eta_i^\ell$ and relate the resulting covariance to $V_4^{(\phi)}$ of Banta et al.~\cite{banta2024structures} in the MLP limit $(\alpha,\eps)=(0,1)$.

\subsection{Exact Update and Derivation of the Approximate Recursion}\label{sec:appx-exact}

Under $\eta_i^\ell\mid\phi^\ell\sim\mathcal{N}(0,\hatQ^\ell)$, the empirical kernel $G^\ell_{ab}:=\frac{1}{n}\sum_i\phi_i^\ell(a)\phi_i^\ell(b)$ updates exactly as $G^{\ell+1} = \alpha^2 G^\ell + \alpha\eps\,H^\ell + \eps^2 J^\ell$, where $H^\ell$ is the cross-term and $J^\ell$ is the increment Gram matrix.
Their conditional covariance tensors are
\begin{align}
  n\,\mathrm{Cov}(H^\ell_{ab},H^\ell_{cd}\mid\cF_\ell)
  &=: \Sigma_H(G^\ell,\hatQ^\ell)_{ab,cd},
  \label{eq:H-conditional-covariance}\\
  n\,\mathrm{Cov}(J^\ell_{ab},J^\ell_{cd}\mid\cF_\ell)
  &=\Omega(\hatQ^\ell)_{ab,cd}.
  \label{eq:J-conditional-covariance}
\end{align}
Here, for any symmetric positive-semidefinite matrix $A$, the conditional-Wishart covariance tensor is
\begin{equation}
  \Omega(A)_{ab,cd}:=A_{ac}A_{bd}+A_{ad}A_{bc}.
  \label{eq:Omega-definition}
\end{equation}
These exact conditional moments do not by themselves close the covariance evolution: the exact update for
$V_4^{(G),\ell}:=\E[v^\ell(v^\ell)^\top]$ also involves
$d^\ell:=\sqrt{n}(\hatQ^\ell-\E[\hatQ^\ell])$.
While (PROJ) closes the cross-covariance of $d^\ell$ and $v^\ell$, it does not close the sigma-kernel self-covariance.
We therefore retain the finite-step defect
\begin{equation}
  \Delta_{\sigma\sigma}^\ell
  :=\mathrm{Cov}(d^\ell,d^\ell)
    -\chi_{K_0^\ell}V_4^{(G),\ell}\chi_{K_0^\ell}^{\top}.
  \label{eq:sigma-self-defect}
\end{equation}

Applying (PROJ) to the cross-covariance terms and evaluating the noise source on the leading background gives the following defect-augmented covariance recursion;
it becomes closed only after specifying a model for $\Delta_{\sigma\sigma}^\ell$.
Applying additionally the delta expansion and (GC1) gives the reference mean correction.
With $A^\ell:=\alpha^2 I+\eps^2\chi_{K_0^\ell}$, these equations are
\begin{align}
  K_0^{\ell+1} &= \alpha^2 K_0^\ell+\eps^2 Q(K_0^\ell), \label{eq:K0-alpha}\\
  V_4^{(G),\ell+1} &\simeq A^\ell\,V_4^{(G),\ell}\,(A^\ell)^\top
  +\alpha^2\eps^2\,\Sigma(K_0^\ell)
  +\eps^4\,\Omega(Q(K_0^\ell))
  +\eps^4\,\Delta_{\sigma\sigma}^\ell,
  \label{eq:V4-alpha}\\
  \Keft^{\ell+1}
  &\simeq \alpha^2 \Keft^\ell
  +\eps^2\chi_{K_0^\ell}[\Keft^\ell]
  +\frac{\eps^2}{2}\,\Hess[K_0^\ell]:V_4^{(G),\ell}.
  \label{eq:K1-alpha}
\end{align}
The $H$-channel source $\Sigma$ appears only when $\alpha\neq0$; the $J$-channel source $\Omega$ and the sigma-kernel self-covariance defect are always present, though both are subleading at order $\eps^4$ in the ResNet continuous-depth limit.

In the ResNet limit $(\alpha,\eps)=(1,\eps\ll1)$, substituting $A^\ell=I+\eps^2\chi_{K_0^\ell}+O(\eps^4)$ into (\ref{eq:V4-alpha}) and (\ref{eq:K1-alpha}) and setting $dt:=\eps^2$ recovers the ODEs (\ref{eq:V4-ode}) and (\ref{eq:K1-ode}) of the main text.
In the MLP limit $(\alpha,\eps)=(0,1)$, with $A^\ell=\chi_{K_0^\ell}$, we obtain $K_0^{\ell+1}=Q(K_0^\ell)$,
\begin{align}
  V_4^{(G),\ell+1}
  &\simeq
  \chi_{K_0^\ell}V_4^{(G),\ell}\chi_{K_0^\ell}^\top
  +\Omega(Q(K_0^\ell))
  +\Delta_{\sigma\sigma}^\ell,\notag\\
  \Keft^{\ell+1}
  &\simeq
  \chi_{K_0^\ell}[\Keft^\ell]
  +\frac{1}{2}\Hess[K_0^\ell]:V_4^{(G),\ell}.
\end{align}

\subsection{Comparison with Banta et al.: Two Variables and the Bridge}\label{sec:appx-banta}

In the MLP case, $\phi_i^{\ell+1}\mid\phi^\ell\sim\mathcal{N}(0,\hatQ^\ell)$ holds exactly.
Banta et al.~\cite{banta2024structures} use this to define the preactivation connected four-point function $n^{-1} V_{4,ab,cd}^{(\phi),\ell} := \langle\phi_i^\ell(a)\phi_i^\ell(b)\phi_j^\ell(c)\phi_j^\ell(d)\rangle_C$ ($i\neq j$).
The corresponding exact identity is $V_{4,ab,cd}^{(\phi),\ell+1} = n\,\mathrm{Cov}(\hatQ^\ell_{ab},\hatQ^\ell_{cd})$.

For the kernel covariance $V_4^{(G)}$ of this paper, the law of total covariance gives the exact bridge identity:
\begin{equation}
  V_{4,ab,cd}^{(G),\ell+1}
  = \E\!\bigl[\hatQ^\ell_{ac}\hatQ^\ell_{bd}+\hatQ^\ell_{ad}\hatQ^\ell_{bc}\bigr]
  + V_{4,ab,cd}^{(\phi),\ell+1}.
  \label{eq:VG-Vphi-bridge}
\end{equation}
At leading order in large $n$, this becomes $V_4^{(G),\ell+1} = \Omega(\bar{K}^{\ell+1}) + V_4^{(\phi),\ell+1} + O(n^{-1})$, showing that $V_4^{(G)}$ is $V_4^{(\phi)}$ plus the conditional-Wishart piece $\E[\Omega(\hatQ^\ell)]$.
Thus, Eqs.~(14) and (15) of Banta et al.\ recursively update $V_4^{(\phi)}$, whereas the corresponding recursion for $V_4^{(G)}$ includes the conditional-Wishart piece.

The generalized block also has an exact conditional-Gaussian layer law in the ResNet case.
The role of (GC0), (PROJ), source-background evaluation, and (GC1) is to close the collective dynamics into a hierarchy in $G^\ell$ alone, not to treat non-Gaussianity of the layer law itself.
Despite the different derivation, the transport parts of the recursions share the same susceptibility operator $\chi$.
At the four-point level, the full leading recursion of Banta et al.\ provides the leading large-width decomposition of the exact identity
\begin{equation}
  V_4^{(\phi),\ell+1}
  =\mathrm{Cov}(d^\ell,d^\ell).
  \label{eq:Banta-full-covariance}
\end{equation}
Indeed, writing
\begin{equation}
  \widetilde\Delta_j^\ell
  :=\sigma(\phi_j^\ell)\sigma(\phi_j^\ell)^\top
    -\E[\sigma(\phi_j^\ell)\sigma(\phi_j^\ell)^\top],
  \label{eq:delta-tilde-def}
\end{equation}
so that $d^\ell=(\Cw/\sqrt n)\sum_j\widetilde\Delta_j^\ell$, gives
\begin{equation}
  \mathrm{Cov}(d^\ell,d^\ell)
  =
  \frac{\Cw^2}{n}\sum_j
  \mathrm{Cov}(\widetilde\Delta_j^\ell,\widetilde\Delta_j^\ell)
  +
  \frac{\Cw^2}{n}\sum_{j\ne k}
  \mathrm{Cov}(\widetilde\Delta_j^\ell,\widetilde\Delta_k^\ell).
  \label{eq:d-self-local-transport}
\end{equation}
The leading Gaussian-background evaluation of the first term is the same-neuron local source $\Cw^2\langle\Delta\Delta\rangle_{K_0}$.
The second term transports the previous-layer connected four-point function through the susceptibility.  Thus the local source and the transport contribution together reconstruct (\ref{eq:Banta-full-covariance}).

In the present kernel-covariance variables,
\begin{equation}
  V_4^{(G),\ell}
  =V_4^{(\phi),\ell}+\Omega(\bar K^\ell)+O(n^{-1}),
  \label{eq:Banta-kernel-variables}
\end{equation}
and
\begin{equation}
  \mathrm{Cov}(d^\ell,d^\ell)
  =\chi_{K_0^\ell}V_4^{(G),\ell}\chi_{K_0^\ell}^{\top}
   +\Delta_{\sigma\sigma}^\ell
  \label{eq:d-self-projected}
\end{equation}
organizes the same sigma-sector covariance in the $G$-based formulation.
Accordingly, $\Delta_{\sigma\sigma}^\ell$ is not the local source of Banta et al.; it is the residual of projecting the full covariance $\mathrm{Cov}(d^\ell,d^\ell)$ onto $V_4^{(G),\ell}$.
The additional conditional-Wishart contribution is $\Omega(Q(K_0^\ell))=\Omega(K_0^{\ell+1})$ in the MLP limit, consistent with the bridge identity (\ref{eq:VG-Vphi-bridge}).
Similarly, the NLO two-point flow of Banta et al.\ (Eqs.~(A.19) and (A.20)) has the same transport-plus-source structure, with the source written in terms of $V_4^{(\phi)}$, whereas the present $G$-based organization writes the reference source in terms of $V_4^{(G)}$.

\subsection{Exact-discrete MLP control}
\label{sec:mlp-control}

To test whether the ResNet closure defects persist without residual carry-over, we repeat the conditional-bias and covariance-residual diagnostics in the matched MLP limit of the generalized block.
Setting $(\alpha,\eps)=(0,1)$ removes the skip term, so the covariance is tested against the exact discrete MLP layer law.
Conditioning on the pre-update layer gives the Rao--Blackwellized target
\begin{equation}
  V_{4,\mathrm{MLP}}^{(G),+}
  =
  \E[\Omega(\hatQ)]
  +n\,\mathrm{Cov}(\hatQ,\hatQ).
  \label{eq:mlp-exact-rb-target}
\end{equation}
The first term is the conditional-Wishart covariance generated in the new layer; concentration gives $\E[\Omega(\hatQ)]=\Omega(Q(K_0))+o(1)$.
For the second term, writing $\hatQ=\Cb+(\Cw/n)\sum_j\sigma(\phi_j)\sigma(\phi_j)^\top$ separates the same-neuron terms $j=k$ from the connected terms $j\ne k$.
The same-neuron contribution is
\begin{equation}
  L(K)_{ab,cd}
  :=
  \Cw^2\,\mathrm{Cov}_{Z\sim\mathcal N(0,K)}
  \left[
    \sigma(Z_a)\sigma(Z_b),
    \sigma(Z_c)\sigma(Z_d)
  \right],
  \label{eq:mlp-local-source}
\end{equation}
while the connected contribution is transported from the previous layer by the susceptibility:
\begin{equation}
  n\,\mathrm{Cov}(\hatQ,\hatQ)
  =
  L(K_0)
  +\chi_{K_0}
   \left[V_4^{(G)}-\Omega(K_0)\right]
   \chi_{K_0}^{\top}
  +o(1).
  \label{eq:mlp-drift-covariance}
\end{equation}
Substituting these two leading large-width expressions into (\ref{eq:mlp-exact-rb-target}) gives the discrete MLP recurrence tested here:
\begin{equation}
  V_{4,\mathrm{model}}^{(G),+}
  =
  \Omega(Q(K_0))
  +L(K_0)
  +\chi_{K_0}
   \left[V_4^{(G)}-\Omega(K_0)\right]
   \chi_{K_0}^{\top}.
  \label{eq:mlp-discrete-model}
\end{equation}
Here $L(K_0)$ is the same-neuron local source, while the last term transports the connected part of the previous-layer covariance.
At each width, the conditional-bias diagnostic is the same
Wishart-subtracted remainder as in (\ref{eq:BQ-Wishart-finite}):
\begin{equation}
  B_{Q,n}^{\mathrm{res}}
  :=
  n\,\E\!\left[\hatQ-Q(G)\right]
  +\frac12\Hess[K_0]:\Omega(K_0).
  \label{eq:mlp-BQ-residual}
\end{equation}
The covariance diagnostic is the difference between the exact target and the discrete model recurrence:
\begin{equation}
  R_{V_4,n}^{\mathrm{MLP}}
  :=
  V_{4,\mathrm{MLP}}^{(G),+}
  -V_{4,\mathrm{model}}^{(G),+}.
  \label{eq:mlp-V4-residual}
\end{equation}
Linear fits in $1/n$ give
\begin{align}
  B_{Q,00}^{\mathrm{res},(\infty)}
  &=0.0045\ [-0.0163,0.0262],&
  B_{Q,01}^{\mathrm{res},(\infty)}
  &=-0.0084\ [-0.0258,0.0072],\notag\\
  R_{V_4,0000}^{\mathrm{MLP},(\infty)}
  &=-0.0005\ [-0.0030,0.0021],&
  R_{V_4,0101}^{\mathrm{MLP},(\infty)}
  &=-0.0021\ [-0.0046,0.0004].
  \label{eq:mlp-control-intercepts}
\end{align}
Here $(\infty)$ denotes the fitted $n\to\infty$ intercept.
All four intervals contain zero.
Figure~\ref{fig:MLP-control-appx} summarizes both MLP extrapolations.

\begin{figure}[t]
  \centering
  \includegraphics[width=\textwidth]{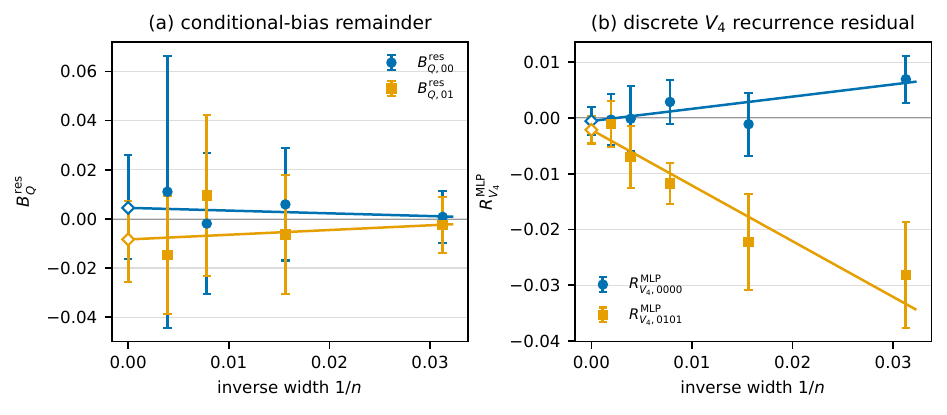}
  \caption{
    \textbf{Exact-discrete MLP control.}
    Left: Wishart-subtracted conditional-bias remainders for matrix components $00$ and $01$.
    Right: total residuals of the discrete $V_4$ recurrence for components $0000$ and $0101$.
    Diamonds at $1/n=0$ denote fitted linear intercepts with bootstrap 95\% intervals.
  }
  \label{fig:MLP-control-appx}
\end{figure}

\section{Technical Results}\label{sec:appx-technical}

\subsection{Exact Gaussian initial projection}\label{sec:appx-proj}

Set $K:=K_0^0$, let $Z\sim\mathcal N(0,K)$, and write $f_{ab}(Z):=\sigma(Z_a)\sigma(Z_b)$.
Because the neurons are independent at initialization, the definitions of $\nu^0$ and $v^0$ give
\begin{equation}
  (C_{\nu v}^0)_{ab,cd}
  =n\,\mathrm{Cov}(S_{ab}^0,G_{cd}^0)
  =\E\left[f_{ab}(Z)(Z_cZ_d-K_{cd})\right].
  \label{eq:appx-Cnuv-init}
\end{equation}
For a smooth activation with the required Gaussian moments, two applications of Gaussian integration by parts~\cite{stein1981estimation} yield
\begin{equation}
  (C_{\nu v}^0)_{ab,cd}
  =
  K_{ce}K_{df}\,
  \E\left[\partial_e\partial_f f_{ab}(Z)\right].
  \label{eq:appx-gaussian-ibp}
\end{equation}
The initial empirical kernel is Wishart~\cite{muirhead1982aspects}, so its scaled covariance is exactly
\begin{equation}
  (V_4^0)_{ef,cd}
  =K_{ec}K_{fd}+K_{ed}K_{fc}.
  \label{eq:appx-wishart-V4}
\end{equation}
Gaussian covariance differentiation in the symmetrized convention of (\ref{eq:chi-def}) gives
\begin{align}
  (\chi_K V_4^0)_{ab,cd}
  &=
  \frac{\Cw}{2}
  \E\left[\partial_e\partial_f f_{ab}(Z)\right]
  (K_{ec}K_{fd}+K_{ed}K_{fc})\notag\\
  &=
  \Cw K_{ce}K_{df}
  \E\left[\partial_e\partial_f f_{ab}(Z)\right]
  =(\Cw C_{\nu v}^0)_{ab,cd},
  \label{eq:appx-PROJ-init}
\end{align}
where the second line uses symmetry under $e\leftrightarrow f$.
Substitution into (\ref{eq:projection-defect}) proves $\Delta_{\mathrm{proj}}^0=0$ exactly at finite $n$.

\subsection{Degenerate Gaussian covariance}\label{sec:appx-singular}

Let a positive-semidefinite covariance have rank $r$ and spectral decomposition $\cG=U_r\Lambda_rU_r^\top$, where $\Lambda_r>0$.
The associated Gaussian is a measure on $\operatorname{Im}\cG$, defined by
\begin{equation}
  \int f(\eta)\,d\mu_{\cG}(\eta)
  :=\int_{\mathbb R^r}
  f(U_ry)\,
  \frac{\exp[-\frac12y^\top\Lambda_r^{-1}y]}
       {(2\pi)^{r/2}(\det\Lambda_r)^{1/2}}\,dy .
  \label{eq:singular-gaussian-measure}
\end{equation}
Equivalently, its quadratic form is $\eta^\top\cG^+\eta$, its normalization uses $\operatorname{pdet}\cG=\det\Lambda_r$, and it carries the support constraint $P_{\ker\cG}\eta=0$.
The determinant cancellation in Section~\ref{sec:block} is therefore performed on this $r$-dimensional support.
Complete $U_r$ to an orthogonal matrix $(U_r,U_0)$, with $\operatorname{Im}U_0=\ker\cG$.
For $\Delta\phi_i^\ell:=\phi_i^{\ell+1}-\phi_i^\ell$, the response-field integral is
\begin{align}
  &\int\frac{d^N\hat\phi_i^\ell}{(2\pi)^N}
  \exp\left[
    i(\hat\phi_i^\ell)^\top\Delta\phi_i^\ell
    -\frac{\eps^2}{2}(\hat\phi_i^\ell)^\top\cG\hat\phi_i^\ell
  \right]\notag\\
  &\quad=
  \delta^{(N-r)}(U_0^\top\Delta\phi_i^\ell)
  \frac{\exp\left[-(\Delta\phi_i^\ell)^\top
    \cG^+\Delta\phi_i^\ell/(2\eps^2)\right]}
  {(2\pi\eps^2)^{r/2}[\operatorname{pdet}\cG]^{1/2}} .
  \label{eq:singular-response-integral}
\end{align}
Thus the null response modes impose the support constraint, while the range modes generate the Gaussian density on $\operatorname{Im}\cG$.
Accordingly, (\ref{eq:block-kernel}) is not an ambient-space Lebesgue density when $\cG$ is singular.

\section{Numerical Protocols and Robustness Tests}
\label{sec:appx-validation}

\subsection{Numerical protocols and convergence checks}
\label{sec:numerical-protocols}

Table~\ref{tab:numerical-protocols} collects the simulation settings used in the validation studies.
Empirical endpoint comparisons use the last recorded pre-update state, $\ell=799$ ($t=1.9975$); the post-update state at $\ell=800$ ($t=2$) is used only for theory-only convergence checks.
\begin{table}[htbp]
  \centering
  \caption{
    Numerical protocols and corresponding results.
    Paired ensembles use common Gaussian references as control variates.
  }
  \label{tab:numerical-protocols}
  \scriptsize
  \begin{tabular}{
    >{\raggedright\arraybackslash}p{1.65cm}
    >{\raggedright\arraybackslash}p{1.8cm}
    >{\raggedright\arraybackslash}p{1.35cm}
    >{\raggedright\arraybackslash}p{1.0cm}
    >{\raggedright\arraybackslash}p{4.0cm}
    >{\raggedright\arraybackslash}p{2.4cm}}
    \hline
    Study & Sector & Width(s) & $\eps$ & Sampling and purpose & Reported in \\
    \hline
    Baseline
      & $\tanh$, $N=4$
      & $64$
      & $0.05$
      & $5\times10^6$ samples; $G$-only matrix baseline and source validation
      & Figs.~\ref{fig:K0-V4}, \ref{fig:K1-main} \\
    Residual decomposition
      & $\tanh$, $N=4$
      & $64$
      & $0.05$
      & Three seeds with $10^4$ samples for trajectories; independent
        five-seed endpoint aggregate for defect localization
      & Fig.~\ref{fig:V4-decomposition} \\
    Wishart / $B_Q$
      & $\tanh$, $N=4$
      & $64$
      & $0.05$
      & Eight seeds with $10^4$ samples; conditional-bias decomposition
      & Fig.~\ref{fig:BQ-decomposition-appx} \\
    Matrix width
      & $\tanh$, $N=4$
      & $32$--$256$
      & $0.05$
      & Six widths, eight seeds, $5000$ samples; paired Rao--Blackwellized defect fits
      & Fig.~\ref{fig:V4-width-appx} \\
    Finite step
      & $\tanh$, $N=4$
      & $64$
      & $0.05$--$0.10$
      & Three step sizes, eight seeds, $5000$ samples; $R_\eps$ scaling
      & Fig.~\ref{fig:V4-step-appx} \\
    MLP control
      & $\tanh$, $N=4$
      & $32$--$512$
      & $1$
      & Exact-discrete $\alpha=0$ control; eight seeds per width
      & Fig.~\ref{fig:MLP-control-appx} \\
    Scalar Hermite
      & $\tanh$, $N=1$
      & $16$--$256$
      & $0.05$
      & Paired five-seed ensembles; $10^4$ samples for $K_1,U_1$ and $5000$
        for $V_4$; trajectory and width-extrapolation tests
      & Figs.~\ref{fig:hermite-k1}, \ref{fig:hermite-v4}; Table~\ref{tab:width-fits} \\
    Robustness
      & $\tanh/\operatorname{erf}$, $N=4,8$
      & $32$--$128$
      & $0.05$--$0.10$
      & Eight seeds with $5000$ samples; parameter, width, and step-size tests
      & Fig.~\ref{fig:robustness} \\
    Softplus
      & $\beta=0.5,1,2$, $N=1$
      & $16$--$256$
      & $0.05$
      & Paired eight-seed ensembles with $5000$ samples; mode localization and
        width extrapolation
      & Fig.~\ref{fig:softplus-all-mode} \\
    \hline
  \end{tabular}
\end{table}

Table~\ref{tab:hermite-convergence} gives the scalar Hermite endpoint values under truncation- and quadrature-order refinement.
The production quadrature order is 560; lowering it to 480 changes $K_1$, $U_1$, and $V_4$ by at most $2.4\times10^{-12}$, $5.4\times10^{-12}$, and $1.5\times10^{-12}$, respectively, over the full trajectory.
The propagated covariance remains positive semidefinite in every run.

\begin{table}[t]
  \centering
  \caption{
    \textbf{Convergence of the scalar Hermite hierarchy.}
    Predictions are evaluated at the last empirical comparison point, $\ell=799$ ($t=1.9975$).
    Here $p_{\max}=2R$ is the maximum retained Hermite degree and $N_{\mathrm{GH}}$ is the Gauss--Hermite quadrature order.
  }
  \label{tab:hermite-convergence}
  \begin{tabular}{ccrrr}
    \hline
    $p_{\max}$ & $N_{\mathrm{GH}}$ & $K_1$ & $U_1$ & $V_4$ \\
    \hline
    24 & 560 & $-0.440603$ & $-0.157696$ & $48.392030$ \\
    36 & 560 & $-0.449477$ & $-0.161985$ & $48.395732$ \\
    48 & 560 & $-0.451394$ & $-0.163092$ & $48.396328$ \\
    60 & 560 & $-0.451903$ & $-0.163430$ & $48.396456$ \\
    80 & 480 & $-0.452099$ & $-0.163579$ & $48.396496$ \\
    80 & 560 & $-0.452099$ & $-0.163579$ & $48.396496$ \\
    \hline
  \end{tabular}
\end{table}

\subsection{Conditional-Wishart bias decomposition}
\label{sec:conditional-wishart-bias-appx}

Figure~\ref{fig:BQ-decomposition-appx} resolves the measured finite-width conditional bias into the leading reference term $B_Q^{\mathrm W}(K_0)$ and the remainder $B_{Q,n}^{\mathrm{res}}$ defined in~(\ref{eq:BQ-Wishart-finite}).
The reference term cancels the initial Wishart tadpole.
At later depths the diagonal remainder becomes increasingly negative before tending toward a plateau, while the off-diagonal remainder is smaller and nearly constant after its initial transient.
The quoted endpoint intervals in (\ref{eq:BQ-res-endpoint}) are Student-$t$ intervals across eight independent seeds and refer only to $n=64$.

\begin{figure}[t]
  \centering
  \includegraphics[width=\textwidth]{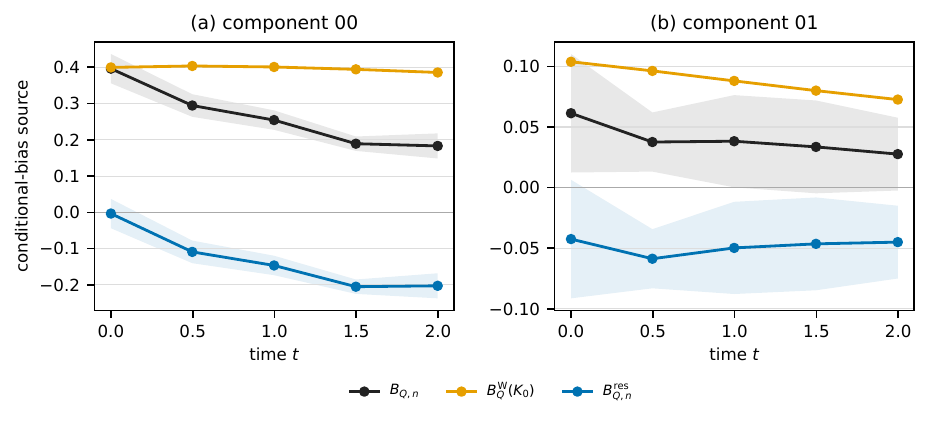}
  \caption{
    \textbf{Conditional-bias decomposition at width $n=64$.}
    The measured $B_{Q,n}$ is compared with the leading conditional-Wishart reference term $B_Q^{\mathrm W}(K_0)$ and the non-Wishart remainder $B_{Q,n}^{\mathrm{res}}$ for components $00$ and $01$.
    Bands are two-sided Student-$t$ 95\% intervals across eight independent seeds.
  }
  \label{fig:BQ-decomposition-appx}
\end{figure}

\subsection{Matrix defect estimators and width extrapolation}
\label{sec:matrix-defect-estimators}

To separate the matrix closure defects identified in Section~\ref{sec:V4-numerics} from Monte Carlo noise and finite-step corrections, we combine a paired Rao--Blackwellized estimator with width and step-size extrapolations.
Figures~\ref{fig:V4-width-appx}--\ref{fig:V4-step-appx} report the resulting tests.
For each seed, samples and layers are kept paired.
Starting from the exact finite-step covariance identity (\ref{eq:V4-exact-discrete}), conditional averaging of the Gaussian innovation gives
\begin{equation}
  \widehat V_4^{\ell+1,\mathrm{RB}}
  =
  \widehat V_4^\ell
  +\eps^2\left(
    \widehat T_{\mathrm{exact}}^\ell+
    \widehat\Sigma_{\mathrm{mic}}^\ell
  \right)
  +\eps^4\left[
    n\,\widehat{\mathrm{Cov}}(\hatQ^\ell,\hatQ^\ell)
    +\widehat{\Omega(\hatQ^\ell)}
  \right].
  \label{eq:RV4-RB-full}
\end{equation}
For the $M$ paired samples within one seed, we use the componentwise empirical covariance
\begin{equation}
  \widehat{\mathrm{Cov}}(X,Y)_{A,B}
  :=
  \frac{1}{M}\sum_{s=1}^{M}
  \left(X_{s,A}-\overline X_A\right)
  \left(Y_{s,B}-\overline Y_B\right),
  \qquad
  \overline X_A:=\frac{1}{M}\sum_{s=1}^{M}X_{s,A}.
  \label{eq:empirical-covariance}
\end{equation}
The first $O(\eps^2)$ term in (\ref{eq:RV4-RB-full}) is the exact transport tensor
\begin{equation}
  \left(\widehat T_{\mathrm{exact}}^\ell\right)_{ab,cd}
  :=
  n\,\widehat{\mathrm{Cov}}\!\left(\hatQ_{ab}^\ell,G_{cd}^\ell\right)
  +
  n\,\widehat{\mathrm{Cov}}\!\left(G_{ab}^\ell,\hatQ_{cd}^\ell\right).
  \label{eq:T-exact-estimator}
\end{equation}
It is the cross term generated by the conditional mean update $\E[G^{\ell+1}-G^\ell\mid\mathcal F_\ell]=\eps^2\hatQ^\ell$.
The other $O(\eps^2)$ term, $\widehat\Sigma_{\mathrm{mic}}$, is obtained by averaging the right-hand side of (\ref{eq:H-cov}) over the same samples.
The empirical Wick contraction in the $O(\eps^4)$ term is
\begin{equation}
  \left[\widehat{\Omega(\hatQ^\ell)}\right]_{ab,cd}
  :=
  \frac{1}{M}\sum_{s=1}^{M}
  \left(
    \hatQ_{s,ac}^\ell\hatQ_{s,bd}^\ell
    +
    \hatQ_{s,ad}^\ell\hatQ_{s,bc}^\ell
  \right).
  \label{eq:empirical-Omega}
\end{equation}
This is the sample average of $\Omega(\hatQ_s^\ell)$, rather than
$\Omega(\overline{\hatQ^\ell})$.

Substitution into the residual definition (\ref{eq:RV4-estimator}) gives
\begin{equation}
\begin{aligned}
  \widehat R_{V_4}^{\ell,\mathrm{RB}}
  ={}&
  \widehat T_{\mathrm{exact}}^\ell
  -\chi_{K_0^\ell}\widehat V_4^\ell
  -\widehat V_4^\ell\chi_{K_0^\ell}^{\top}
  +\widehat\Sigma_{\mathrm{mic}}^\ell-\Sigma(K_0^\ell) \\
  &+\eps^2\left[
    n\,\widehat{\mathrm{Cov}}(\hatQ^\ell,\hatQ^\ell)
    +\widehat{\Omega(\hatQ^\ell)}
  \right].
\end{aligned}
  \label{eq:RB-residual-expanded}
\end{equation}
Adding and subtracting $Q(G^\ell)$ in the transport part of (\ref{eq:RB-residual-expanded}) gives the two empirical contributions corresponding to (\ref{eq:Rdelta-Rtau}):
\begin{equation}
  \widehat R_\delta^\ell
  :=
  n\,\widehat{\mathrm{Cov}}
  \left(\hatQ^\ell-Q(G^\ell),G^\ell\right)
  +
  n\,\widehat{\mathrm{Cov}}
  \left(G^\ell,\hatQ^\ell-Q(G^\ell)\right),
  \label{eq:Rdelta-estimator}
\end{equation}
and
\begin{equation}
  \widehat R_\tau^\ell
  :=
  n\,\widehat{\mathrm{Cov}}\left(Q(G^\ell),G^\ell\right)
  +
  n\,\widehat{\mathrm{Cov}}\left(G^\ell,Q(G^\ell)\right)
  -\chi_{K_0^\ell}\widehat V_4^\ell
  -\widehat V_4^\ell\chi_{K_0^\ell}^{\top}.
  \label{eq:Rtau-estimator}
\end{equation}
The finite-step contribution is
\begin{equation}
  \widehat R_\eps^\ell
  :=
  \eps^2\left[
    n\,\widehat{\mathrm{Cov}}(\hatQ^\ell,\hatQ^\ell)
    +\widehat{\Omega(\hatQ^\ell)}
  \right].
  \label{eq:Reps-estimator}
\end{equation}
Combining the transport mismatch, source mismatch, and finite-step contribution gives
\begin{equation}
  \widehat R_{V_4}^{\ell,\mathrm{RB}}
  =
  \widehat R_\delta^\ell+\widehat R_\tau^\ell
  +\left[
    \widehat\Sigma_{\mathrm{mic}}^\ell-\Sigma(K_0^\ell)
  \right]
  +\widehat R_\eps^\ell .
  \label{eq:RB-four-term-decomposition}
\end{equation}

We fit $Y(n)=Y_\infty+a_1/n$ to the six widths $n=32,48,64,96,128,256$.
Fits use inverse-variance weights and 5000 seed-level bootstrap replicates; no sample-level bootstrap is used.
Figure~\ref{fig:V4-width-appx} shows the resulting endpoint extrapolations.
Figure~\ref{fig:V4-step-appx} separately verifies that the finite-step remainder decreases as $O(\eps^2)$.

\begin{figure}[t]
  \centering
  \includegraphics[width=\textwidth]{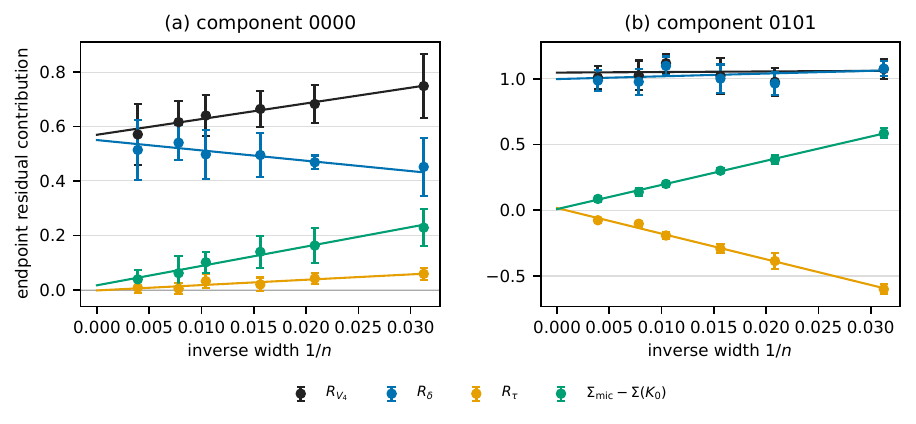}
  \caption{
    \textbf{Width extrapolation of the matrix covariance residual.}
    Endpoint values for components $0000$ and $0101$ are plotted against $1/n$.
    The solid lines are linear fits to the six displayed widths.
    Here $R_\delta=\mathrm{Cov}(\delta,v)+\mathrm{Cov}(v,\delta)$ is the conditional contribution and $R_\tau=\mathrm{Cov}(\tau,v)+\mathrm{Cov}(v,\tau)$ the Taylor/background contribution; their sum is the symmetrized projection defect, $\Delta_{\mathrm{proj}}+\Delta_{\mathrm{proj}}^\top=R_\delta+R_\tau$.
  }
  \label{fig:V4-width-appx}
\end{figure}

\begin{figure}[t]
  \centering
  \includegraphics[width=\textwidth]{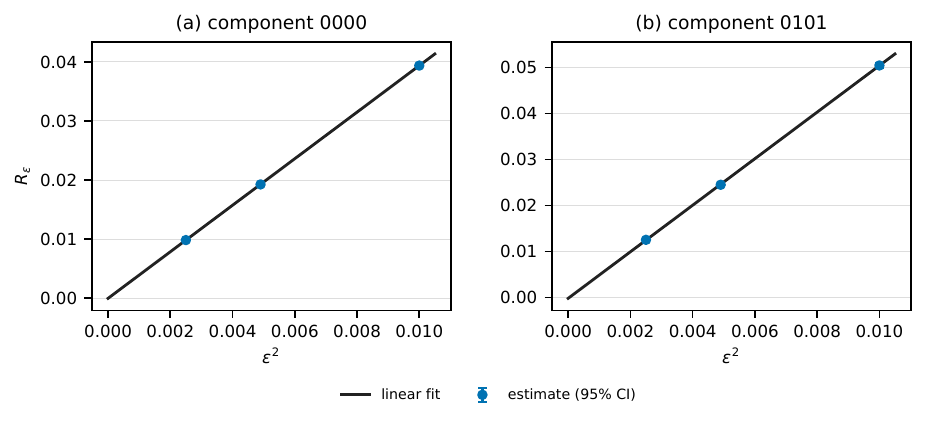}
  \caption{
    \textbf{Finite-step scaling of the matrix covariance residual.}
    The endpoint remainder $R_\eps$ is linear in $\eps^2$ over $\eps=0.05,0.07,0.10$ for components $0000$ and $0101$, consistent with $R_\eps=O(\eps^2)$.
  }
  \label{fig:V4-step-appx}
\end{figure}

\subsection{Scalar tanh width extrapolation}

In the scalar sector, let $q^\ell:=K_0^\ell$ denote the deterministic background satisfying $q^{\ell+1}=q^\ell+\Delta t\,Q(q^\ell)$, and define $G_n^\ell:=n^{-1}\sum_i(\phi_i^\ell)^2$ and $S_n^\ell:=n^{-1}\sum_i\tanh^2(\phi_i^\ell)$.
We extrapolate the rescaled finite-width coefficients, with $c_0$ defined in~(\ref{eq:hermite-background}),
\begin{align}
  K_{1,n}^\ell&:=n\bigl(\E[G_n^\ell]-q^\ell\bigr),&
  U_{1,n}^\ell&:=n\bigl(\E[S_n^\ell]-c_0(q^\ell)\bigr),&
  V_{4,n}^\ell&:=n\,\operatorname{Var}(G_n^\ell),
  \label{eq:finite-width-rescaled-coefficients}
\end{align}
which approach $K_1^\ell$, $U_1^\ell$, and $V_4^\ell$ with $O(n^{-1})$ corrections.
Thus, for example, $\operatorname{Var}(G_n^\ell)=V_4^\ell/n+o(n^{-1})$ still vanishes as $n\to\infty$.

For each finite-width scalar trajectory, we evolve a paired Gaussian reference from the same initial sample and with the same standard-normal increments:
\begin{align}
  \phi_{i,\mathrm{ref}}^{\ell+1}
  &=\phi_{i,\mathrm{ref}}^\ell
    +\eps\sqrt{Q(q^\ell)}\,z_i^\ell,
  &\phi_{i,\mathrm{ref}}^0&=\phi_i^0,\notag\\
  G_{n,\mathrm{ref}}^\ell
  &:=\frac{1}{n}\sum_i(\phi_{i,\mathrm{ref}}^\ell)^2,
  &S_{n,\mathrm{ref}}^\ell
  &:=\frac{1}{n}\sum_i\tanh^2(\phi_{i,\mathrm{ref}}^\ell).
  \label{eq:paired-gaussian-reference}
\end{align}
Since $\phi_{i,\mathrm{ref}}^\ell\sim\mathcal N(0,q^\ell)$, $\E[G_{n,\mathrm{ref}}^\ell]=q^\ell$ and $\E[S_{n,\mathrm{ref}}^\ell]=c_0(q^\ell)$ exactly.
For $M$ independent paired trajectories, write $\overline X:=M^{-1}\sum_{I=1}^M X^{(I)}$ and $s_M^2(X):=(M-1)^{-1}\sum_{I=1}^M(X^{(I)}-\overline X)^2$.
We then use
\begin{align}
  \widehat U_{1,\mathrm{CV}}^\ell(n)
  &:=n\left(\overline{S_n^\ell}
             -\overline{S_{n,\mathrm{ref}}^\ell}\right),\notag\\
  \widehat V_{4,\mathrm{CV}}^\ell(n)
  &:=n\left[s_M^2(G_n^\ell)-s_M^2(G_{n,\mathrm{ref}}^\ell)\right]
     +2(q^\ell)^2,\notag\\
  \widehat K_{1,\mathrm{CV}}^{\ell+1}
  &:=\widehat K_{1,\mathrm{CV}}^\ell
   +\Delta t\,\Cw\,\widehat U_{1,\mathrm{CV}}^\ell,
  \qquad \widehat K_{1,\mathrm{CV}}^0=0.
  \label{eq:paired-control-variates}
\end{align}
Here $2(q^\ell)^2=n\,\operatorname{Var}(G_{n,\mathrm{ref}}^\ell)$, and the last line is the exact discrete mean recursion.
Thus the $V_4$ estimator subtracts the fluctuating sample variance of the paired Gaussian reference and restores its known expectation, preserving $n\,\operatorname{Var}(G_n^\ell)$ while canceling shared Gaussian noise.
All three estimators are unbiased; the common random numbers further reduce their variance~\cite{glasserman2004monte}.
Because $\widehat K_{1,\mathrm{CV}}$ is obtained by integrating $\widehat U_{1,\mathrm{CV}}$ through the exact mean recursion, the $K_1$ and $U_1$ comparisons are statistically dependent; the $K_1$ trajectory instead tests the accumulated consistency of the predicted sigma-kernel source.

At $n\in\{16,24,32,48,64,96,128,256\}$, the primary fit is $Y(n)=Y_\infty+a_1/n$; adding $a_2/n^2$ provides a model-form sensitivity check.
Each width is weighted by the inverse estimated variance of its seed mean.
Bootstrap replicates~\cite{efron1993introduction} resample seeds independently within each width and refit using the original estimated weights.
Table~\ref{tab:width-fits} reports percentile 95\% intervals at the matched pre-update layer 799.
\begin{table}[htbp]
  \centering
  \caption{
    Hermite and $G$-only predictions compared with variance-reduced large-width extrapolations of the rescaled coefficients at layer 799 ($t=1.9975$).
    Each empirical fit is shared by the two theoretical predictions in its row.
  }
  \label{tab:width-fits}
  \scriptsize
  \begin{tabular}{lcccc}
    \hline
    Quantity & Hermite & $G$-only & Linear fit [95\% CI] & Quadratic fit [95\% CI] \\
    \hline
    $K_1$ & $-0.45210$ & $-1.02528$ & $-0.44978\;[-0.45492,-0.44457]$ & $-0.45335\;[-0.46122,-0.44543]$ \\
    $U_1$ & $-0.16358$ & $-0.55347$ & $-0.16323\;[-0.16534,-0.16111]$ & $-0.16459\;[-0.16776,-0.16145]$ \\
    $V_4$ & $48.3965$ & $47.7168$ & $48.4447\;[48.3609,48.5232]$ & $48.5155\;[48.4034,48.6214]$ \\
    \hline
  \end{tabular}
\end{table}

\subsection{Direct tanh and erf robustness}

We also tested the predictions directly beyond the baseline configuration.
Figure~\ref{fig:robustness} reports the signed Monte Carlo--theory differences for $K_0$ and the $G$-only $V_4$ solution under nonzero $\Cb$, lower $\Cw$, lower initial correlation, $N=8$, and the smooth activation $\sigma(x)=\operatorname{erf}(x/\sqrt2)$.
Across these changes, the endpoint error of $K_0$ remains sub-percent, while the $G$-only $V_4$ differences remain at the few-percent level.
Full settings and uncertainty conventions are given in Table~\ref{tab:numerical-protocols}.

\begin{figure}[t]
  \centering
  \includegraphics[width=\textwidth]{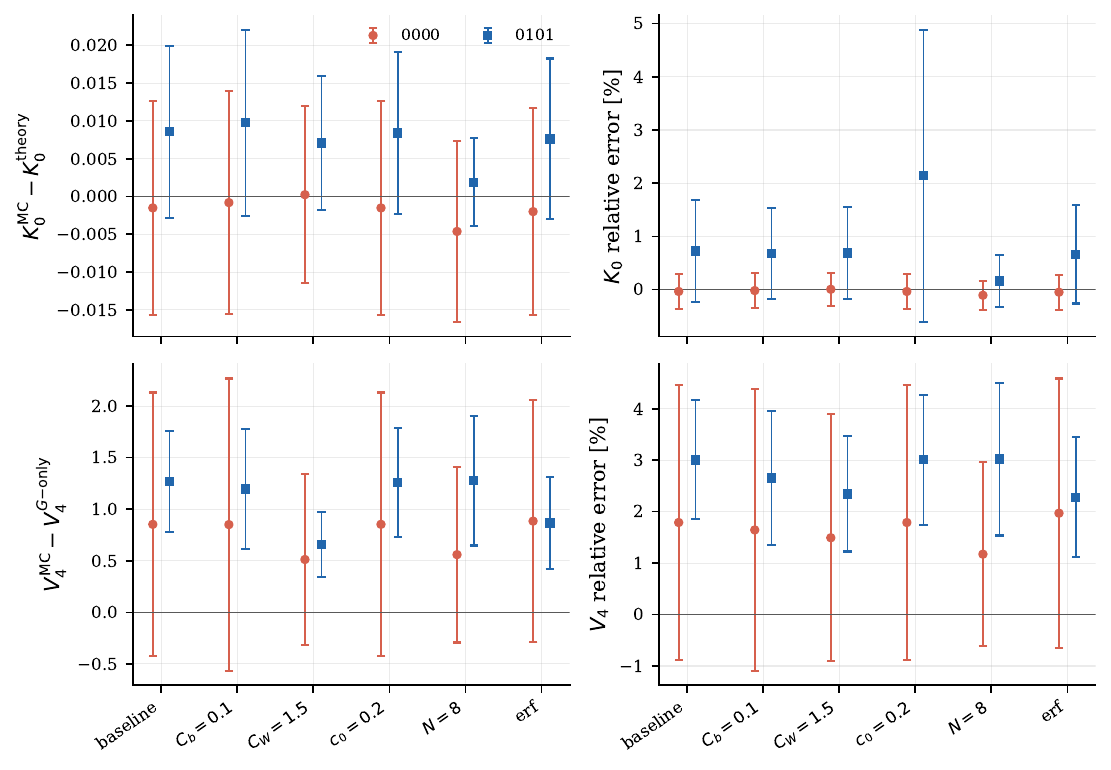}
  \caption{
    \textbf{Direct tanh and erf validation.}
    Signed absolute differences (left) and relative differences (right) between Monte Carlo and theory are shown for $K_0$ (top) and the $G$-only $V_4$ prediction (bottom), for components $0000$ and $0101$;
    relative differences are normalized by the absolute theory value.
    The erf case uses the same $\Cw=2$ as the tanh baseline.
    Error bars are two-sided Student-$t$ 95\% intervals across independent seeds.
  }
  \label{fig:robustness}
\end{figure}

\subsection{Softplus stress test and additional collective modes}

To probe a smooth family approaching ReLU without identifying smoothness with closure accuracy, we use
\begin{equation}
  \sigma_\beta(x)=\operatorname{softplus}_\beta(x)
  =\frac{1}{\beta}\log(1+e^{\beta x}),
  \qquad
  \sigma_\beta(x)\longrightarrow \operatorname{ReLU}(x)
  \quad(\beta\to\infty).
  \label{eq:softplus-definition}
\end{equation}
At $\beta=0.5,1,2$, respectively, we choose $\Cw\simeq0.373,0.753,0.974$ so that the initial diagonal kernel drift agrees with the tanh reference.
The leading kernel remains accurate in the direct finite-width tests, but the $G$-only covariance does not.
This shows that differentiability permits the NLO expansion but does not imply that $G$ is a sufficient collective variable.

The missing transport can be localized directly.
Writing
\begin{equation}
 \sigma_\beta^2(\sqrt q\,x)
 =c_0(q)+\sum_{k\geq1}c_k(q)h_k(x),
  \label{eq:softplus-hermite-expansion}
\end{equation}
we estimated each contribution to $\Cw\operatorname{Cov}(\nu,v)$ from standardized empirical Hermite modes at the endpoint defined in Section~\ref{sec:numerical-protocols}.
The degree-2 contribution accounts for approximately $71\%$, $75\%$, and $77\%$ at $\beta=0.5,1,2$, while the odd modes account for $28\%$, $25\%$, and $23\%$.
The retained degree-40 sum reconstructs the measured transport to within $0.2\%$ in all three cases.
The non-degree-2 transport is therefore localized predominantly in odd Hermite sectors.

This diagnosis suggests a minimal scalar correction.
We extend the construction of Section~\ref{sec:hermite} from even degrees $h_2,h_4,\ldots$ to all degrees $1\leq k\leq p_{\max}$.
The degree-2 row still reconstructs the scalar kernel fluctuation, while the additional rows transport the sigma-kernel fluctuation entering its covariance.
Explicitly, the leading all-mode covariance update uses
\begin{equation}
  \mathcal A_{ks}
  =\rho^k\delta_{ks}
   +\delta_{k2}\frac{\Delta t\,\Cw}{\sqrt2\,q^+}c_s(q),
  \qquad
  D_{ks}=\delta_{ks}\bigl(1-\rho^{2k}\bigr),
  \qquad k,s=1,\ldots,p_{\max},
  \label{eq:softplus-all-mode-update}
\end{equation}
in $C^+=\mathcal A C\mathcal A^\top+D$.
At $p_{\max}=40$, all three predictions lie inside both the linear and quadratic bootstrap 95\% intervals and reduce the absolute $G$-only discrepancy by approximately $97\%$ relative to the linear extrapolation centers.
Simulation settings are collected in Table~\ref{tab:numerical-protocols}.

The all-mode predictions are stable under quadrature- and mode-order refinement, and the propagated covariance remains positive definite.
These results establish a constructive scalar $V_4$ correction for the tested smooth Softplus family; they do not establish a matrix-valued Softplus closure or an empirical Softplus $K_1$ result.
Figure~\ref{fig:softplus-all-mode} compares the two scalar predictions with the width-extrapolated intervals and displays the measured mode-resolved contributions.

\begin{figure}[t]
  \centering
  \includegraphics[width=\textwidth]{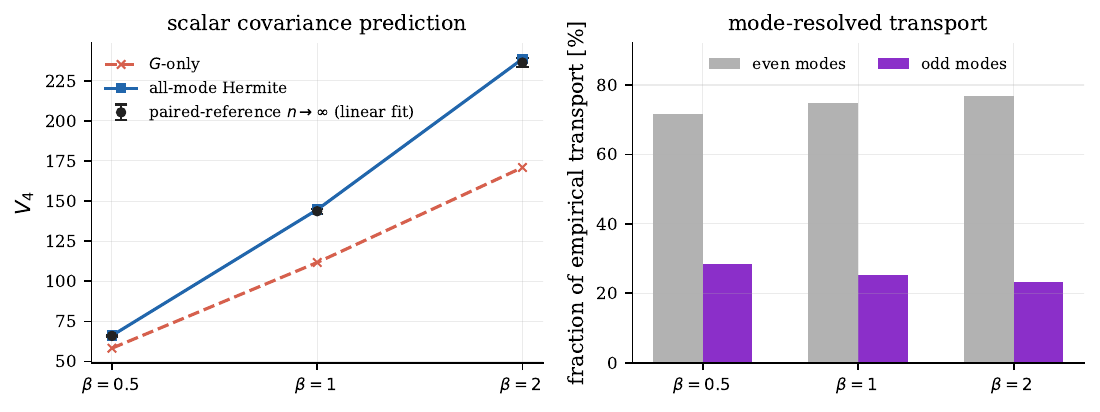}
  \caption{
    \textbf{Softplus stress test and all-mode scalar correction.}
    Left: at the final recorded pre-update endpoint (layer 799), paired-reference large-width linear extrapolations of the rescaled covariance coefficient, with bootstrap 95\% intervals, are compared with the $G$-only and all-mode Hermite predictions.
    Right: at the same endpoint and width $n=64$, the measured scalar sigma-kernel transport is separated into odd and even modes, with the even contribution dominated by degree~2.
    The initial diagonal kernel drift is matched to the tanh reference at each $\beta$.
  }
  \label{fig:softplus-all-mode}
\end{figure}

The Softplus result does not define a ReLU NLO limit.
Although the exact conditional-Gaussian block law and finite-width kernel recursion require only the relevant activation moments and therefore include ReLU, the derivative sector is nonuniform because $\sigma_\beta''(x)=\beta e^{\beta x}/(1+e^{\beta x})^2$ concentrates at the origin as $\beta\to\infty$.
Consequently, a quantitative ReLU prediction for $V_4$ or $K_1$ requires a nonsmooth, ReLU-adapted closure.

\section{Notation and Symbol Guide}\label{sec:notation-guide}

Tables~\ref{tab:notation-core}, \ref{tab:notation-corrections}, and
\ref{tab:notation-hermite} collect the core model, finite-width correction,
and scalar Hermite and MLP-control notation, respectively.

\begin{table}[htbp]
  \centering
  \caption{Core model and exact finite-width notation}
  \label{tab:notation-core}
  \scriptsize
  \renewcommand{\arraystretch}{1.15}
  \begin{tabular}{p{3.50cm}p{6.45cm}p{2.65cm}}
  \hline
  Symbol & Meaning & Definition \\
  \hline
  $\eps,\ t=\eps^2\ell;\ \Cw,\Cb;\ \rho_{\mathrm{in}}$ &
  Residual step and time; weight and bias variances; initial input correlation &
  Sections~\ref{sec:setup}, \ref{sec:V4-numerics} \\
  $\phi_i^\ell,\eta_i^\ell,\cF_\ell$ &
  Preactivation, residual increment, and pre-update filtration &
  Eqs.~(\ref{eq:resnet-block})--(\ref{eq:increment-gaussian}); Section~\ref{sec:recursion} \\
  $G^\ell,S^\ell,\hatQ^\ell$ &
  Empirical kernel, empirical sigma-kernel, and exact conditional increment covariance $\hatQ^\ell=\Cb+\Cw S^\ell$ &
  Eqs.~(\ref{eq:kernel-update}), (\ref{eq:mean-sigma}), (\ref{eq:H-cov}) \\
  $E_2(K),Q(K)$ &
  Gaussian sigma-kernel map and drift kernel &
  Eqs.~(\ref{eq:E2-def}), (\ref{eq:Q-def}) \\
  $\chi_K,\Hess[K]$ &
  Linear susceptibility and second Fr\'echet derivative $D^2Q(K)$ &
  Eqs.~(\ref{eq:chi-def}), (\ref{eq:hess-def}) \\
  $H^\ell,J^\ell$ &
  Cross increment and increment Gram matrix in the exact kernel update &
  Eqs.~(\ref{eq:H-def}), (\ref{eq:J-def}) \\
  $\bar K^\ell,K_0^\ell$ &
  Ensemble mean $\bar K^\ell=\E[G^\ell]$ and leading deterministic background &
  Eqs.~(\ref{eq:mean-sigma}), (\ref{eq:K0-ode}) \\
  \hline
  \end{tabular}
\end{table}

\begin{table}[htbp]
  \centering
  \caption{Finite-width corrections, closures, and diagnostics.
  Unadorned $V_4$ in the main text denotes $V_4^{(G)}$.}
  \label{tab:notation-corrections}
  \scriptsize
  \renewcommand{\arraystretch}{1.15}
  \begin{tabular}{p{3.50cm}p{6.45cm}p{2.65cm}}
  \hline
  Symbol & Meaning & Definition \\
  \hline
  $K_1,\Kmic,\Keft$ &
  NLO mean coefficient, finite-width estimate, and $G$-only prediction &
  Eqs.~(\ref{eq:finite-width-representation-statistics}), (\ref{eq:exact-mean-source-definitions}), (\ref{eq:K1-ode}) \\
  $K_{1,\mathrm{src}},K_{1,R}^{\mathrm H}$ &
  Measured-source reconstruction and scalar Hermite prediction &
  Eqs.~(\ref{eq:K1-source-reconstruction}), (\ref{eq:hermite-reconstruction}) \\
  $\Uex,\Umod,U_{1,R}^{\mathrm H}$ &
  Exact finite-width source, $G$-only modeled source, and scalar Hermite prediction &
  Eqs.~(\ref{eq:exact-mean-source-definitions}), (\ref{eq:Umod-def}), (\ref{eq:hermite-reconstruction}) \\
  $V_4^{(G)},V_4^{(\phi)},\Omega(K)$ &
  Kernel covariance, connected preactivation four-point tensor, and conditional-Wishart covariance &
  Eqs.~(\ref{eq:VG-Vphi-bridge}), (\ref{eq:bridge-leading}), (\ref{eq:Omega-main}) \\
  $C_{\nu v},\Sigma_{\mathrm{mic}},\Sigma(K_0)$ &
  Sigma-kernel/kernel cross-covariance, microscopic noise source, and its background evaluation &
  Eqs.~(\ref{eq:V4-exact-discrete}), (\ref{eq:Sigma-mic-def}), (\ref{eq:Sigma-def}) \\
  $\Delta_Q,\delta,\tau,\Delta_{\mathrm{proj}}$ &
  Sample-wise defect, scaled defect, Taylor/background residual, and one-sided projection defect &
  Eqs.~(\ref{eq:DeltaQ-def}), (\ref{eq:delta-tau}), (\ref{eq:projection-defect}) \\
  $R_\delta,R_\tau$ &
  Conditional and Taylor/background parts of the symmetrized projection defect &
  Eq.~(\ref{eq:Rdelta-Rtau}) \\
  $R_\eps,R_{V_4}$ &
  Finite-step remainder and total covariance-equation residual &
  Eqs.~(\ref{eq:V4-exact-discrete}), (\ref{eq:RV4-def}) \\
  $B_{Q,n},B_Q$ &
  Finite-width conditional bias and its large-width limit &
  Eqs.~(\ref{eq:BQ-def}), (\ref{eq:GC1}) \\
  $B_Q^{\mathrm W},B_{Q,n}^{\mathrm{res}},B_Q^{\mathrm{res}}$ &
  Wishart term and corresponding finite- and large-width remainders &
  Eqs.~(\ref{eq:BQ-Wishart-limit}), (\ref{eq:BQ-Wishart-finite}) \\
  $v,\nu,d,\zeta,r_J$ &
  Centered kernel, sigma-kernel, drift, cross-increment, and increment-Gram fluctuations &
  Eqs.~(\ref{eq:v-exact}), (\ref{eq:d-nu-exact}) \\
  \hline
  \end{tabular}
\end{table}

\begin{table}[htbp]
  \centering
  \caption{Scalar Hermite and MLP-control notation}
  \label{tab:notation-hermite}
  \scriptsize
  \renewcommand{\arraystretch}{1.15}
  \begin{tabular}{p{3.50cm}p{6.45cm}p{2.65cm}}
  \hline
  Symbol & Meaning & Definition \\
  \hline
  $\alpha,\Delta_{\sigma\sigma},L(K)$ &
  Generalized-block skip coefficient, sigma-kernel self-covariance defect, and same-neuron MLP source &
  Eqs.~(\ref{eq:sigma-self-defect}), (\ref{eq:mlp-local-source}) \\
  $q,q^+,\Delta t,\rho,\gamma$ &
  Scalar background before and after one block, residual step, carry-over coefficient, and variance coupling &
  Eqs.~(\ref{eq:hermite-background}), (\ref{eq:hermite-standardized-parameters}) \\
  $h_k,c_r(q),a_r,\xi_r$ &
  Hermite basis, activation coefficients, empirical modes, and centered mode fluctuations &
  Eqs.~(\ref{eq:hermite-orthogonality})--(\ref{eq:hermite-scaling}) \\
  $\boldsymbol m_R,C$ &
  NLO mode means and leading mode covariance &
  Eqs.~(\ref{eq:hermite-scaling}), (\ref{eq:hermite-cov}), (\ref{eq:hermite-mean}) \\
  $\mathcal A,D,\boldsymbol\omega,\boldsymbol b(C)$ &
  Mode transport, innovation covariance, martingale innovation, and NLO mean source &
  Eqs.~(\ref{eq:hermite-martingale})--(\ref{eq:hermite-mean}) \\
  $V_{4,R}^{\mathrm H},K_{1,R}^{\mathrm H},U_{1,R}^{\mathrm H}$ &
  Reconstructed scalar covariance, mean, and source predictions &
  Eq.~(\ref{eq:hermite-reconstruction}) \\
  $\Delta_{\mathrm{proj},R}^{\mathrm H},B_{Q,R}^{\mathrm H}$ &
  Reconstructed projection and conditional-bias diagnostics &
  Eqs.~(\ref{eq:hermite-projection-defect}), (\ref{eq:BQ-hermite}) \\
  $R,\ p_{\max}=2R$ &
  Hermite truncation order and maximum retained even degree &
  Eq.~(\ref{eq:hermite-modes}) \\
  \hline
  \end{tabular}
\end{table}

\bibliographystyle{JHEP}
\bibliography{refs}

\end{document}